\pdfoutput=1
\documentclass[a4paper,USenglish,cleveref, autoref, thm-restate]{lipics-v2021}

\hideLIPIcs  %
\nolinenumbers %

\title{Optimal Rates for Robust Stochastic Convex Optimization}


\author{Changyu Gao}{University of Wisconsin-Madison, Madison, WI, USA}{ustcgcy@gmail.com}{}{}
\author{Andrew Lowy}{University of Wisconsin-Madison, Madison, WI, USA}{alowy@wisc.edu}{}{}
\author{Xingyu Zhou}{Wayne State University, Detroit, MI, USA}{xingyu.zhou@wayne.edu}{}{} %
\author{Stephen J. Wright}{University of Wisconsin-Madison, Madison, WI, USA}{swright@cs.wisc.edu}{}{} %

\authorrunning{C. Gao, A. Lowy, X. Zhou, and S.\,J. Wright}
\Copyright{Changyu Gao, Andrew Lowy, Xingyu Zhou, and Stephen J. Wright}


\category{} %

\funding{Research of CG, AL, and SW was supported by NSF Awards 2023239 and 2224213 and AFOSR Award FA9550-21-1-0084. XZ is supported in part by NSF CNS-2153220 and CNS-2312835.} 

\acknowledgements{
Changyu would like to thank Shuyao Li for helpful discussions.
}

\usepackage{booktabs} %

\usepackage{bbm}
\usepackage{algorithm}
\usepackage{algpseudocode}

\newtheorem{assumption}[theorem]{Assumption}

\usepackage[textsize=tiny]{todonotes}

\newcommand{\E}{\mathbb{E}}
\newcommand{\R}{\mathbb{R}}

\newcommand{\goodset}{S_{\text{good}}}

\newcommand{\argmin}{\operatornamewithlimits{arg\,min}}
\newcommand{\trace}{\operatorname{tr}}
\newcommand{\cov}{\operatorname{Cov}}
\newcommand{\srank}{\operatorname{r}}

\newcommand{\citet}[1]{\cite{#1}}
\newcommand{\citep}[1]{\cite{#1}}

\EventEditors{Mark Bun}
\EventNoEds{1}
\EventLongTitle{6th Symposium on Foundations of Responsible Computing (FORC 2025)}
\EventShortTitle{FORC 2025}
\EventAcronym{FORC}
\EventYear{2025}
\EventDate{June 4--6, 2025}
\EventLocation{Stanford University, CA, USA}
\EventLogo{}
\SeriesVolume{329}
\ArticleNo{9}

\begin{document}
\maketitle

\begin{abstract}
  Machine learning algorithms in high-dimensional settings are highly susceptible to the influence of even a small fraction of structured outliers, making robust optimization techniques essential. In particular, within the \(\epsilon\)-contamination model, where an adversary can inspect and replace up to an \(\epsilon\)-fraction of the samples, a fundamental open problem is determining the optimal rates for robust stochastic convex optimization (SCO) under such contamination. We develop novel algorithms that achieve \emph{minimax-optimal} excess risk (up to logarithmic factors) under the \(\epsilon\)-contamination model. Our approach improves over existing algorithms, which are not only suboptimal but also require stringent assumptions, including Lipschitz continuity and smoothness of individual sample functions.
  By contrast, our optimal algorithms do not require these stringent assumptions, assuming only population-level smoothness of the loss.
  Moreover, our algorithms can be adapted to handle the case in which the covariance parameter is unknown, and can be extended to nonsmooth population risks via convolutional smoothing.
  We complement our algorithmic developments with a \emph{tight information-theoretic lower bound} for robust SCO.\@
\end{abstract}

\section{Introduction}
Machine learning models are increasingly deployed in security-critical applications, yet they remain vulnerable to data manipulation.
A particular threat is data poisoning, where adversaries deliberately insert malicious points into training data to degrade model performance \citep{biggio2012poisoning}.
Even in non-adversarial settings, naturally occurring outliers can significantly impact learning algorithms, especially in high-dimensional settings.
These challenges motivate our study of optimization algorithms for training machine learning models in the presence of outliers, both natural and adversarial.

Motivation for our work traces to Tukey's pioneering research on robust estimation \citep{tukey1960survey}.
Recent breakthroughs have produced efficient algorithms for high-dimensional robust estimation under the \(\epsilon\)-contamination model, where an adversary can arbitrarily replace up to an \(\epsilon\)-fraction of the samples.
Notable advances include polynomial-time algorithms for robust mean estimation in high dimensions \citep{diakonikolas2017being, diakonikolas2019robust}.
See \citet{diakonikolas2019recent} for a comprehensive survey of recent developments in high-dimensional robust estimation.

These developments in robust estimation naturally lead to a fundamental question: \emph{Can we solve {\em stochastic optimization\/} problems, under the \(\epsilon\)-contamination model?}
Stochastic optimization is used in machine learning to find the parameter that minimizes the population risk using training samples.
We focus specifically on robust stochastic optimization with convex objective functions whose gradients exhibit bounded covariance, a standard assumption in robust mean estimation \citep{diakonikolas2020outlier}.
While our goal aligns with the classical use of stochastic convex optimization in minimizing population risk, the presence of adversarial contamination introduces significant new challenges.

Prior research in robust optimization has concentrated primarily on narrow domains.
One line of work focuses on robust linear regression \citep{klivans2018efficient, diakonikolas2019efficient, cherapanamjeri2020optimal}.
While \citep{jambulapati2021robust,prasad2020robust} have explored general problems, they focus on robust {\em regression}.
To our best knowledge, SEVER \citep{diakonikolas2019sever} is the only work that considers general stochastic optimization problems.
However, this approach has several limitations that restrict the applicability of SEVER.
First, it focuses only on achieving dimension-independent error due to corruption, with only a suboptimal sample complexity.
Second, the results for SEVER depend on several stringent assumptions, including Lipschitzness and smoothness conditions on \emph{individual sample} functions.
\emph{Because of these limitations, optimal excess risk bounds for robust stochastic convex optimization, and under what conditions they can be achieved, remain unknown.}

In this work, we develop efficient algorithms for robust stochastic convex optimization that achieve \textbf{optimal excess risk bounds} (up to logarithmic factors) under the \(\epsilon\)-contamination model.
Notably, \cref{alg:robust_net} assumes only the smoothness of the population risk.
Moreover, we prove a matching lower bound to show the minimax-optimality of our algorithms.

\subsection{Problem Setup and Motivation}\label{sec:problem-setup}
\textbf{Notation.} 
For a vector \(v \in \R^d\), \(\| v \|\) denotes the \(\ell_2\) norm of \(v\).
For a matrix \(A \in \R^{d \times d}\), \(\| A \|\) denotes the spectral norm of \(A\).
For symmetric matrices \(A\) and \(B\), we write  \(A \preceq B\) if \(B - A\) is positive semidefinite (PSD).
We use \(\tilde O\) and \(\tilde \Omega\) to hide logarithmic factors in our bounds.

Let \(\mathcal{W} \subset \R^d\) be a closed convex set.
Consider a distribution \(p^{\ast}\) over functions \(f : \mathcal{W} \to \R\). Stochastic optimization aims to
find a parameter vector \(w^{\ast} \in \mathcal{W}\) minimizing the population risk \(\overline{f}(w) := \E_{f \sim p^{\ast}}[f(w)]\).
For example, function \(f\) can take the form of a loss function \(f_x(w)\) dependent on the data point \(x\), and the data distribution on \(x\) induces the function distribution \(p^{\ast}\).
In robust stochastic optimization, some data samples may be corrupted.
Following \citep{diakonikolas2019sever}, we  adopt the strong \(\epsilon\)-contamination model, which allows the adversary to replace up to \(\epsilon\) fraction of samples.
\begin{definition}[\(\epsilon\)-contamination model]\label{def:eps-contam}
  Given \(\epsilon > 0\) and a distribution \(p^{\ast}\) over functions \(f : \mathcal{W} \to \R\), data is generated as follows:
  first, \(n\) clean samples \(f_1, \ldots, f_{n}\) are drawn from \(p^{\ast}\).
  An \emph{adversary} is then permitted to examine the samples and replace up to \(\epsilon n\) of them with arbitrary samples.
  The algorithm is subsequently provided with this modified set of functions, which we refer to as {\em \(\epsilon\)-corrupted\/} samples (with respect to \(p^{\ast}\)).
\end{definition}
This model is strictly stronger than the Huber contamination model \citep{huber1964robust}, in which the samples are drawn from a mixture of the clean and adversarial distributions of the form \(p^{\ast} = (1-\epsilon)p + \epsilon q\), where \(p\) is the clean distribution and \(q\) is the adversarial distribution.

Our objective is to develop an efficient algorithm that minimizes the population risk \(\overline{f}(w)\), even when the data is \(\epsilon\)-corrupted.
The following is assumed throughout the paper.
\begin{assumption}\label{assump:conv_loss}
  ~\begin{enumerate}
    \item \(\mathcal{W} \subset \R^d\) is a compact convex set with diameter \(D\), that is, \(\sup_{w, w' \in \mathcal{W}} \|w - w'\| \le D\).
    \item \(f\) is differentiable almost surely. The population risk \(\overline{f}(w)\) is convex.
    \item The regularity condition holds \(\E_{f \sim p^*} [\nabla f(w)] = \nabla \overline f(w)\).\footnote{This technical assumption allows us to exchange the expectation and the gradient. See discussions in \cref{sec:regularity-condition}}
  \end{enumerate}
\end{assumption}

We also assume in  most results that the gradients of the functions have bounded covariance as in \citet{diakonikolas2019sever}, which is a typical assumption used in robust mean estimation.
\begin{assumption}\label{assump:cov}
  There is \(\sigma>0\) such that for all \(w \in \mathcal{W}\) and all unit vectors \(v\), we have \(\mathbf{E}_{f \sim p^*}[(v \cdot (\nabla f(w) - \nabla \overline{f}(w)))^2] \le \sigma^2\).
\end{assumption}
An equivalent form of this assumption is that for every \(w \in \mathcal{W}\),
the covariance matrix of the gradients, defined by \(\Sigma_w := \E_{f \sim p^*}[(\nabla f(w) - \nabla \overline{f}(w))(\nabla f(w) - \nabla \overline{f}(w))^T]\)
  satisfies \(\Sigma_w \preceq \sigma^2 I\).
  (See \cref{sec:discussion_assump} for a proof.)

We will additionally assume that the population risk \(\overline{f}(w)\) satisfies certain properties, or that certain properties are satisfied almost surely for functions \(f\) from distribution \(p^*\), as needed. %

To our best knowledge, SEVER \citep{diakonikolas2019sever} is the only work that studies robust stochastic optimization for general convex losses.
While SEVER focuses on finding approximate critical points,
our work focuses on minimizing the population risk \(\overline{f}(w)\), and we measure the performance of our algorithm in terms of the excess risk \(\overline{f}(\hat{w}) - \min_w \overline{f}(w)\), where \(\hat{w}\) is the output of the algorithm.

We remark that SEVER also derives excess risk bounds.
To contrast with SEVER, we decompose the excess risk of a stochastic optimization algorithm as follows\footnote{
  We omit the term due to optimization error that depends on the number of iterations of the algorithm, since it will be dominated by the other terms when we run the optimization algorithm for a sufficient number of iterations.
}:
\[
  \text{Excess risk} = \text{Error due to corruption} + \text{Statistical error},
\]
where ``error due to corruption'' refers to the error due to the presence of corruption in the data, while ``statistical error'' denotes the error that accrues even when there is no corruption.
SEVER \citep{diakonikolas2019sever} focuses only on the error due to corruption.
The statistical error term is implicit in their requirement on the sample complexity \(n\), that is,
\[
  \text{Excess risk} = \text{Error due to corruption}, \; \text{if } n \ge \, \text{[sample complexity]}.
\]
Specifically, they design a polynomial-time algorithm that achieves \(O(D \sigma \sqrt{\epsilon})\)  error due to corruption term  for \(n = \tilde \Omega\left(\frac{dL^2}{\epsilon \sigma^2} + \frac{dL^4}{\sigma^4}\right)\), provided that \(f - \overline f\) is \(L\)-Lipschitz and \(\beta\)-smooth almost surely for \(f \in p^*\), and that \(f\) is smooth almost surely.
(Their analysis has an incorrect sample complexity result. See \cref{app:sever-fix} for our correction.)
This sample complexity can be huge (even infinite), as some functions in the distribution may have a very large (possibly unbounded) Lipschitz constant. 
Moreover, SEVER implicitly requires \(f\) to be smooth almost surely.

Consider functions of the form \(f_x(w) = -\tfrac{1}{2} x \cdot \|w\|^2 \) for \(w\) such that \(\|w\| \le D\), where \(x \sim P\) for a probability distribution \(P\) with bounded mean and variance but with unbounded values, e.g.\ the normal distribution. 
We have \(\nabla f_x(w) = -x \cdot w\).  
Since \(x\) is unbounded, the worst-case Lipschitz parameter and smoothness of \(f\) are both infinite.
However, the population risk \(\overline f(w) = -\tfrac{1}{2}\|w\|^2 \cdot \mathbf{E}[x]\) is smooth and Lipschitz.
This example demonstrates that the assumptions in SEVER that assume properties uniformly for individual functions \(f \sim p^*\) can be too stringent. In this paper, we aim to answer the following question:
\begin{quote}
    Can we design computationally efficient algorithms that achieve the \emph{optimal excess risk} for robust SCO, under much milder conditions?
\end{quote}
We give positive answers to this question and summarize our contributions below.

\subsection{Our Contributions}
\begin{enumerate}
  \item Optimal Rates for Robust SCO (\cref{sec:robust_net}):
We develop algorithms that achieve the following minimax-optimal (up to logarithmic factors) excess risk:
\[\overline f(\hat w_T) - \min_{w \in \mathcal{W}} \overline f(w) = \tilde O \left(
    D
    \left(\sigma \sqrt{\epsilon} + \sigma \sqrt{\frac{d \log (1/\tau)}{n}}\, \right)
  \right).
\]
Compared with SEVER, we achieve the same error due to corruption \(O(D \sigma \sqrt{\epsilon})\) provided \(n = \tilde \Omega(d/\epsilon)\), a significant improvement in sample complexity.\footnote{
  We remark that in excess risk bounds, \(\tilde O\) always hides logarithmic factors only in the statistical error term, and the robust term is always \(O(D \sigma \sqrt{\epsilon})\).
}
\item Much Weaker Assumptions for Robust SCO:
\cref{alg:robust_net} achieves the optimal rates while only assuming the smoothness of the population risk, which is significantly weaker than the assumptions used in SEVER.\@ By contrast, SEVER requires \(f-\overline{f}\) to have bounded worst-case Lipschitz and smoothness parameter, and that individual functions \(f\) are smooth almost surely.

\item Handling unknown \(\sigma\) and extensions to nonsmooth case: Simple adaptations allow our algorithm to handle the case in which the covariance parameter \(\sigma\) is unknown.
We also extend our algorithm to nonsmooth population risks using convolutional smoothing.
The resulting algorithm achieves the minimax-optimal excess risk.

\item A Matching Lower Bound for Robust SCO:
We show a matching lower bound,
demonstrating that our excess risk bound is {\em minimax-optimal\/} (up to logarithmic factors). Consequently, our sample complexity for achieving the error due to corruption \(O(D \sigma \sqrt{\epsilon})\) is also minimax-optimal.
\item A Straightforward Algorithm for Robust SCO (\cref{sec:robust_gd}): \cref{alg:robust_gd} is an elementary algorithm that achieves the same optimal excess risk,
with more stringent assumptions compared to \cref{alg:robust_net}.\@ Our approach builds on the “many-good-sets” assumption, which SEVER briefly introduced without providing a concrete analysis.
\end{enumerate}

Our results might be surprising, as net-based approaches (e.g., uniform convergence) typically suffers from suboptimal error. Our results, however, imply that the net-based approach can indeed achieve the optimal excess risk under the \(\epsilon\)-contamination model. We discuss this further in \cref{sec:compare-covariance-variance}. 
A high-level summary of our results appears in \cref{tab:comparison}.

\begin{table}[htb]
  \centering
  \begin{tabular}{@{}cccc@{}}
    \toprule
    \textbf{Algorithm} & \textbf{Assumptions} & \textbf{Excess Risk} & \textbf{Sample Complexity} \\ \midrule
    SEVER \citep{diakonikolas2019sever} &
    \begin{tabular}[c]{@{}c@{}}
      1.\ \(f - \overline{f}\) is  \(L\)-Lipschitz a.s. \\  2.\ \(f\) is \(\beta\)-smooth a.s.
    \end{tabular}
    & suboptimal & \(\tilde \Omega\left(\frac{dL^2}{\epsilon \sigma^2} + \frac{dL^4}{\sigma^4}\right)\) \\ \cmidrule(lr){2-2}
    \cref{alg:robust_net}
    & \(\overline f\) is \(\bar\beta\)-smooth {\em or} \(\bar{L}\)-Lipschitz.
    & optimal
    & \(\tilde{\Omega}\left(d / \epsilon \right)\) \\ \cmidrule(lr){2-2}
    \cref{alg:robust_gd}
    & \begin{tabular}[c]{@{}c@{}}
1. \ \(f - \overline{f}\) is  \(L\)-Lipschitz a.s. and \(\beta\)-smooth a.s. \\ 2. \(\overline f\) is \(\bar\beta\)-smooth {\em or} \(\bar{L}\)-Lipschitz.
    \end{tabular}
    & optimal & \(\tilde{\Omega}\left(d / \epsilon \right)\) \\ \bottomrule
  \end{tabular}
  \caption{Comparison of assumptions, rates, and sample complexity of SEVER and our two algorithms. The parameters \(\beta\), \(L\), etc are all assumed to be finite. All algorithms assume \cref{assump:conv_loss}, and bounded covariance of the gradients, that is, the covariance matrix \(\Sigma_w\) satisfies \(\Sigma_w \preceq \sigma^2 I\) for all \(w\). Optimality is up to logarithmic factors.
  For the case when  \(\overline f\) is nonsmooth but Lipschitz (see \cref{sec:nonsmooth}), the excess risk is optimal (up to logarithmic factors) under the noncentral moment assumption.
  }\label{tab:comparison}
\end{table}

\section{Revisiting SEVER}\label{sec:sever}
In this section, we revisit SEVER \citep{diakonikolas2019sever} to motivate our work.
Below we fix the corruption parameter \(\epsilon\) and the covariance boundedness parameter \(\sigma > 0\).
Given \(\epsilon\)-corrupted function samples \(f_1, \ldots, f_n\), we say a subset of functions is ``good'' with respect to \(w\) if their sample mean and covariance at \(w\) are close to those of the true distribution, as defined below.

\begin{definition}[``Good'' set]\label{def:good-set}
  We say a set \(S_{\text{good}} \subseteq [n]\) with \(|S_{\text{good}}| \geq (1-\epsilon)n\) is ``good'' w.r.t. \(w\)  if the functions \(\{f_i\}_{i \in S_{\text{good}}}\) satisfy the following,
  \begin{equation}
    \begin{aligned}
      \left\| \frac{1}{|S_{\text{good}}|} \sum_{i \in S_{\text{good}}} \big(\nabla f_i(w) - \nabla \bar f(w)\big)\big(\nabla f_i(w) - \nabla \bar f(w)\big)^T \right\| & \leq O(\sigma^2), \\
      \left\|\frac{1}{|S_{\text{good}}|} \sum_{i \in S_{\text{good}}} (\nabla f_i(w) - \nabla \bar f(w)) \right\| & \leq O(\sigma\sqrt{\epsilon}).
    \end{aligned}
  \end{equation}
\end{definition}
A ``good'' set w.r.t.\ \(w\) allows us to robustly estimate the gradient at \(w\).
SEVER requires the existence of a set that is uniformly good for all \(w\), which we refer to as the ``uniform-good-set'' assumption.
\begin{assumption}[{``Uniform good set'', \cite[Assumption B.1]{diakonikolas2019sever}}]\label{assump:uniform-good-set}
  There exists a set \(S_{\text{good}} \subseteq [n]\)
  with \(|S_{\text{good}}| \geq (1-\epsilon)n\)
  such that \(S_{\text{good}}\) is ``good'' w.r.t. \(w\), for all \(w \in \mathcal{W}\).
\end{assumption}
SEVER operates through an iterative filtering framework built around a black-box learner. Its core algorithm consists of three main steps: (1) The black-box learner processes the current set of functions to find approximate critical points. (2) A filtering mechanism identifies and removes outlier functions. (3) The algorithm updates its working set with the remaining functions.
This process repeats until convergence. Crucially, SEVER's theoretical guarantees rely on its ``uniform-good-set'' assumption. Without this assumption (as opposed to ``many-good-sets'' assumption introduced later), the set of ``good'' functions can change at each iteration, potentially preventing the iterative filtering process from converging.

We argue that the ``uniform-good-set'' assumption can be too strong.
Recall that SEVER requires a sample complexity of \(n = \tilde \Omega\left(\frac{dL^2}{\epsilon \sigma^2} + \frac{dL^4}{\sigma^4}\right)\).
When \(n = \tilde \Omega(d/\epsilon)\), the ``uniform-good-set'' assumption can no longer be guaranteed to hold.
In contrast, the ``many-good-sets'' assumption introduced below is weaker, and aligns with the general framework of robustly estimating gradients in each iteration.

SEVER also assumes the existence of a black box approximate learner.
\begin{definition}[\(\gamma\)-approximate learner]
  A learning algorithm \(\mathcal{L}\) is called \(\gamma\)-approximate if, for any functions
  \(f_1,\ldots,f_m : \mathcal{W} \to \mathbb{R}\), each bounded below on a closed domain \(\mathcal{H}\),
  the output \(w\) of \(\mathcal{L}\) is a \(\gamma\)-approximate critical point of
  \(\hat f(x) := \frac{1}{m}\sum_{i=1}^m f_i(x)\),
  that is, there exists \(\delta > 0\) such that for all unit vectors \(v\) where \(w + \delta v \in \mathcal{W}\), we have that \(v \cdot \nabla \hat f(w) \geq -\gamma\).
\end{definition}
\begin{remark}
  We remark that the existence of a \(\gamma\)-approximate learner implies that the learner can find a \(\gamma\)-approximate critical point of any function \(f\) by choosing \(f_1 = \ldots = f_m = f\). To our best knowledge, any polynomial-time algorithm that finds approximate critical points requires smoothness of the objective.
  Therefore, SEVER does not apply to problems where some functions in the distribution are nonsmooth. For example, consider a distribution \(p^*\) consisted of two functions with equal probability, \(h + g\) and \(h - g\), where \(h\) is smooth but \(g\) is nonsmooth. The population risk is smooth, but the individual functions are not.
\end{remark}

In the appendix of \citet{diakonikolas2019sever}, the authors consider the ``many-good-sets'' assumption, an alternative weaker assumption that allows the good set to depend on the point \(w\).

\begin{assumption}[{``Many good sets'', \cite[Assumption D.1]{diakonikolas2019sever}}]\label{assump:many-good-sets}
  For every \(w\), there exists a set \( S_{\text{good}}(w)  \subseteq [n]\)
  with \(|S_{\text{good}}(w)| \geq (1-\epsilon)n\)
  such that \(S_{\text{good}}(w)\) is ``good'' with respect to \(w\). 
\end{assumption}
We remark that the ``many-good-sets'' assumption allows us to do robust gradient estimation in each iteration.
The SEVER paper mentions (without going into detail) that under the ``many-good-sets'' assumption, projected gradient descent can be used to find a \(O(\sigma \sqrt{\epsilon})\)-approximate critical point. 
It is unclear that under what conditions ``many-good-sets'' assumption can be satisfied, and no excess risk bound or sample complexity is provided.

In this paper, we utilize a further relaxed assumption stated below, which only requires the existence of good sets at points in a fine net of the domain.
For these purposes, we define a \(\xi\)-net of \(\mathcal{W}\) (for some small \(\xi>0\)) to be a set \(\mathcal{C}\)  such that for any \(w \in \mathcal{W}\), there exists \(w' \in \mathcal{C}\) with \(\|w - w'\| \le \xi\).
\begin{assumption}[{``Dense good sets''}]\label{assump:dense-good-sets}
  For a given \(\xi>0\), there exists a \(\xi\)-net \(\mathcal{C}\)  of the domain \(\mathcal{W}\)
  such that for every \(w \in \mathcal{C}\), there exists a set \(S_{\text{good}}(w)  \subseteq [n]\) with \(|S_{\text{good}}(w)| \geq (1-\epsilon)n\)
  such that \(S_{\text{good}}(w)\) is ``good'' with respect to \(w\).
 
\end{assumption}
When the  ``dense good sets'' assumption holds, we can approximate the gradient at any point in the domain \(\mathcal{W}\) by robustly estimating the gradient at the nearest point in the net.
The approximation error will be small, provided that the population risk is smooth and the net is fine enough. 
(The parameter \(\xi\) will depend on \(\sigma\), as we will see in \cref{alg:robust_net}.)
This relaxed assumption allows us to circumvent the technical difficulties of dealing with infinitely many \(w\), thus removing the requirements of uniform Lipschitzness and smoothness of \(f - \overline{f}\) for all \(f\) that are used in SEVER.\@
As a consequence, we are able to achieve the same corruption error as SEVER with a \emph{significantly reduced} sample complexity. The next section presents our algorithm that achieves this result.

\section{Optimal Rates for Robust SCO under Weak Distributional Assumptions}\label{sec:robust_net}

We now present a net-based algorithm that achieves the minimax-optimal excess risk under the weak assumption that the population risk \(\overline{f}\) is smooth.

\begin{assumption}\label{assump:smooth-bounded-var-new}
Given the distribution  \(p^*\) over functions \(f: \mathcal{W} \rightarrow \mathbb{R}\) with \(\overline{f} = \mathbf{E}_{f\sim p^*}[f]\), we have that 
    \(\overline{f}\) is \(\bar \beta\)-smooth. %
\end{assumption}

Here, the \(\beta\)-smoothness requirement only applies to the population risk. Each individual function \(f\) can have different smoothness parameters.

We outline our algorithm below, see \cref{alg:robust_net}. The algorithm is based on projected gradient descent with a robust estimator. Here, we treat the robust gradient estimator \(\texttt{RobustEstimator}\) as a black box, which can be any deterministic stability-based algorithm. For completeness, we provide an instantiation of the robust gradient estimator due to \citet{diakonikolas2020outlier}, outlined in \cref{alg:filtering}, which at a high level iteratively filters out points that are ``far'' from the sample mean in a large variance direction. \cref{alg:filtering} runs in polynomial time.

\textbf{The key innovation lies in its gradient estimation strategy.} 
Rather than computing gradients at arbitrary points, it makes use a dense net of the domain \(\mathcal{W}\), estimating the gradient at the current iterate \(w\) by the gradient at the nearest point \(w'\) in the net to \(w\).
The smoothness of the population risk ensures that this approximation remains accurate.
As mentioned at the end of \cref{sec:sever}, this strategy helps us avoid the technical challenges of handling infinitely many \(w\) with a net argument, thereby achieving optimal rates under significantly weaker distributional assumptions compared to SEVER.\@
\begin{algorithm}[H]
  \caption{Net-based Projected Gradient Descent with Robust Gradient Estimator}\label{alg:robust_net}
  \begin{algorithmic}[1]
    \State \textbf{Input:} \(\epsilon\)-corrupted set of functions \(f_1, \ldots, f_n\), stepsize parameters \(\{\eta_t\}_{t \in [T]}\), robust gradient estimator \(\texttt{RobustEstimator}\), \((\sigma \sqrt{\epsilon} / \bar \beta)\)-net of \(\mathcal{W}\) denoted by \( \mathcal{C} \).
    \State Initialize \(w_0 \in \mathcal{W}\) and \(t = 1\).
    \For{\(t \in [T]\)}
    \State Let \(w_{t-1}' := \argmin_{w \in \mathcal{C}} \| w - w_{t-1} \|\) denote the closest point to \(w_{t-1}\) in the net.
    \State Robustly estimate gradient at \(w_{t-1}'\) 
    by running the \texttt{RobustEstimator} on samples \(\left\{ \nabla f_i(w_{t-1}') \right\}_{i=1}^n\), which outputs \(\tilde g_t\).
    \State \(w_t \gets \text{Proj}_{\mathcal{W}}(w_{t-1} - \eta_t \tilde g_t)\).
    \EndFor
    \State \textbf{Output:} \(\hat w_T = \frac{1}{T} \sum_{t=1}^T w_t\).
  \end{algorithmic}
\end{algorithm}

\begin{algorithm}
  \caption{An Instantiation of the Robust Gradient Estimator: Iterative Filtering \citet{diakonikolas2020outlier}}
  \begin{algorithmic}[1]\label{alg:filtering}
    \State \textbf{Input:} \(\epsilon\)-corrupted set \(S \subset \R^d\) of \(n\) samples.
    \State Initialize weight function \(h: S \to \mathbb{R}_{\geq 0}\) with \(h(x) = 1/|S|\) for all \(x \in S\).
    \While{\(\|h\|_1 \ge 1 - 2\epsilon\)}
    \State Compute \(\mu(h) = \frac{1}{\|h\|_1}\sum_{x\in S}h(x)x\).
    \State Compute \(\Sigma(h) = \frac{1}{\|h\|_1}\sum_{x\in S}h(x)(x-\mu(h))(x-\mu(h))'\).
    \State Compute approximate largest eigenvector \(v\) of \(\Sigma(h)\).
    \State Define \(g(x) = |v \cdot (x-\mu(h))|^2\) for all \(x \in S\).
    \State Find largest \(t\) such that \(\sum_{x\in S:g(x)\geq t}h(x) \geq \epsilon\).
    \State Define \(g_{\ge t}(x) =
      \begin{cases}
        g(x) & \text{if } g(x) \geq t \\
        0 & \text{otherwise}
    \end{cases}\).
    \State Let \(m = \max\{g_{\ge t}(x) : x \in S, h(x) \neq 0\}\).
    \State Update \(h(x) \gets h(x)(1-g_{\ge t}(x)/m)\) for all \(x \in S\).
    \EndWhile
    \State \textbf{Output:} \(\mu(h)\).
  \end{algorithmic}
\end{algorithm}

\textbf{Efficient Implementation:}
For implementation efficiency, we propose a grid-based net construction. Let \(\xi = \sigma \sqrt{\epsilon} / \bar \beta\). We use grid points spaced \(\xi/\sqrt{d}\) apart in each dimension, i.e.,
\[\left\{
    \frac{\xi}{\sqrt{d}} \cdot z = \left( \frac{\xi}{\sqrt{d}}
      \cdot z_1, \frac{\xi}{\sqrt{d}}
      \cdot z_2, \ldots, \frac{\xi}{\sqrt{d}}
    \cdot z_d \right) : z = (z_1, z_2, \ldots, z_d) \in \mathbb{Z}^d, \left\|\frac{\xi}{\sqrt{d}} \cdot z \right\|_2 \le D
\right\}\]
to construct a \(\xi\)-net.\footnote{Technically, we can choose a grid spaced \(\xi/\sqrt{4d}\) apart in each dimension, and add additional points to cover the boundary of the feasible set. This would reduce the size of grid points by almost a factor of \(2^d\).}
Given a point \(w\), we can find a net point within \(\xi\) distance in \(O(d)\) time through:
(1) Scaling: Divide \(w\) by \(\xi/\sqrt{d}\).
(2) Rounding: Convert to the nearest integral vector in \(\mathbb{Z}^d\).
(3) Rescaling: Multiply by \(\xi/\sqrt{d}\).

This construction yields a net of size \(|\mathcal{C}| = O\left(D \sqrt{d}/\xi \right)^d\), which is larger than the optimal covering number \(O((D/\xi)^d)\). While this introduces an extra \(\log d\) factor in the excess risk bound (due to union bound over net points), it offers two significant practical advantages: (1) Implicit net: No need to explicitly construct and store the net. (2) Efficient computation: \(O(d)\) time for finding the nearest net point. An exponential-time algorithm that achieves the excess risk without the \(\log d\) factor is described in \cref{app:exponential-time-alg}.

\textbf{Polynomial runtime:} The robust gradient estimator in \cref{alg:filtering} runs in polynomial time, when used with the grid-based construction above.
As can be seen in \cref{proof:robust-net}, the required number of iterations is also polynomial in parameters. Therefore, the algorithm runs in polynomial time overall.

Convergence of \cref{alg:robust_net} is described  in the following result.

\begin{theorem}\label{thm:robust_net}
  Suppose that  \cref{assump:conv_loss}, \cref{assump:cov}, and \cref{assump:smooth-bounded-var-new} hold. 
  There are choices of stepsizes \(\{\eta_t\}_{t=1}^T\) and \(T\) such that, with probability at least \(1-\tau\), we have
  \[\overline f(\hat w_T) - \min_{w \in \mathcal{W}} \overline f(w) = \tilde O\left(
    \sigma D \sqrt{\epsilon} + \sigma D \sqrt{\frac{d \log (1/\tau)}{n}}
  \, \right).\]
  As a consequence, the algorithm achieves excess risk of \(O(D \sigma\sqrt{\epsilon})\) with high probability whenever \(n = \tilde \Omega(d / \epsilon)\).
\end{theorem}
\begin{remark}
  \cref{thm:robust_net} is minimax-optimal (up to logarithmic factors).
  Our sample complexity \(n = \tilde \Omega(d/\epsilon)\), significant improves over the sample complexity of SEVER, which is  \(
    n = \tilde \Omega\left(
      \frac{dL^2}{\epsilon \sigma^2} + \frac{dL^4}{\sigma^4}
    \right)
  \).
\end{remark}

In \cref{app:lower-bound}, we derive the following matching lower bound, showing the minimax-optimality (up to logarithmic factors) of \cref{alg:robust_net}.
The following matching lower bound can be established, showing the minimax-optimality (up to logarithmic factors) of \cref{alg:robust_net}.
\begin{theorem}\label{thm:lower-bound-robust-stochastic-optimization}
  For \(d \ge 140\) and \(n \ge 62500\), there exist a closed bounded set \(\mathcal{W} \subset \mathbb{R}^d\) with diameter at most \(D\) and a distribution \(p^*\) over functions \(f: \mathcal{W} \rightarrow \mathbb{R}\) that satisfy the following: Let \(\overline{f} = \mathbf{E}_{f\sim p^*}[f]\). We have that for every \(w \in \mathcal{W}\) and unit vector \(v\) that \(\mathbf{E}_{f \sim p^*}[(v \cdot (\nabla f(w) - \nabla \overline{f}(w)))^2] \le \sigma^2\). Both \(f\) (almost surely) and \(\overline{f}\) are convex, Lipschitz and smooth. The output \(\hat w\) of any algorithm with access to an \(\epsilon\)-corrupted set of functions \(f_1, \dots, f_n\) sampled from \(p^*\) satisfies the following with probability at least \(1/2\):
  \begin{equation}\label{eq:lower-bound-robust-stochastic-optimization}
    \overline f(\hat w) - \min_{w \in \mathcal{W}} \overline f(w) = \Omega\left(D \sigma \sqrt{\epsilon} + D \sigma \sqrt{\frac{d}{n}} \right).
  \end{equation}
\end{theorem}

\subsection{Proof Sketch of Theorem~\ref{thm:robust_net}~}\label{sec:proof-robust-net}
We defer the full proof to \cref{app:proof_alg1} and sketch the proof below.
In each iteration of \cref{alg:robust_net}, we estimate the gradient at the current iterate \(w\) by applying the robust gradient estimator to its nearest point in the net \(w'\). We can decompose the error as follows:
\[
\|\tilde g(w') - \nabla \overline{f}(w)\| \le \|\tilde g(w') - \nabla \overline{f}(w')\| + \|\nabla \overline{f}(w') - \nabla \overline{f}(w)\|,
\]
where the first term measures the bias of the robust gradient estimator, and the second term is due to the approximation error due to the net.

We will show that there exist good sets for all net (grid) points (cf. \cref{assump:dense-good-sets}) with high probability, so that we can robustly estimate gradients for all points in the net. This gives a bound for the first term in the equation above, whereas the second term can be bounded using smoothness of the population risk \(\overline{f}\).

Once we establish the gradient estimation bias in each iteration (with high probability), we use the projected biased gradient descent analysis framework (\cref{sec:robust_gd}) to establish an upper bound on the excess risk.

\subsection{Handling Nonsmooth but Lipschitz Population Risks}\label{sec:nonsmooth}

We now consider a setting in which  \cref{assump:conv_loss} and  \cref{assump:cov} hold, but  \(\overline f\) is nonsmooth in the sense that \cref{assump:smooth-bounded-var-new} is not satisfied. 
That is, there is no \(\bar\beta<\infty\) such that \(\overline{f}\) is \(\bar\beta\)-smooth. 
We assume instead that \(\overline{f}\) is \(\bar L\)-Lipschitz for some finite \(\bar{L}\).
In this setting, we can use convolutional smoothing and run \cref{alg:robust_net} on the smoothed objective.
Our algorithm works as follows:
\begin{enumerate}
  \item For every index \(i \in [n]\), we independently sample a perturbation \(u_i \sim \mathcal{U}_s\),
  where \(\mathcal{U}_s\) is the uniform distribution over the \(d\)-dimensional \(L^2\)-norm ball of radius \(s\) centered at the origin.
  We replace samples \(\{f_i\}_{i=1}^n\) by the smoothed samples \(\{f_i(\cdot+u_i)\}_{i=1}^n\).
  \item We run \cref{alg:robust_net} on the smoothed samples with \(\beta\) replaced by \(\frac{\bar L \sqrt{d}}{s}\) and \(\sigma\) replaced by \(\sqrt{\sigma^2 + 4 \bar L^2}\).
\end{enumerate}
The modified algorithm has the following convergence guarantees.
\begin{proposition}\label{thm:nonsmooth}
Suppose that \cref{assump:conv_loss} and \cref{assump:cov} hold and that  \(\overline f\) is \(\bar L\)-Lipschitz. 
There are choices of algorithmic parameters
   such that, with probability at least \(1-\tau\), the output \(\hat w_T\) of the modified algorithm satisfies the following excess risk bound:
  \[\overline f(\hat w_T) - \min_{w \in \mathcal{W}} \overline f(w) = \tilde O \left(
      (\sigma + \bar L) D \sqrt{\epsilon} +
      (\sigma + \bar L) D \sqrt{\frac{d \log (1/\tau)}{n}}
  \, \right).\]
\end{proposition}
\begin{remark}
  Compared to the smooth case, the excess risk bound has an extra \(\bar L\) term. Using this result, we can show that under the alternative noncentral moment assumption, that is,
  instead of \cref{assump:cov}, assume that for every \(w \in \mathcal{W}\) and unit vector \(v\), \(\mathbf{E}_{f \sim p^*}[(v \cdot \nabla f(w))^2] \le G^2\), our modified algorithm achieves the following excess risk bound.
\end{remark}

\begin{theorem}\label{thm:nonsmooth_noncentral}
Suppose that \cref{assump:conv_loss} holds, and that for every \(w \in \mathcal{W}\) and unit vector \(v\), we have \(\mathbf{E}_{f \sim p^*}[(v \cdot \nabla f(w))^2] \le G^2\) for some \(G\). 
There are choices of algorithmic parameters
   such that, with probability at least \(1-\tau\), the output \(\hat w_T\) of the modified algorithm satisfies the following excess risk bound:
  \[\overline f(\hat w_T) - \min_{w \in \mathcal{W}} \overline f(w) = \tilde O \left(
      G D \sqrt{\epsilon} +
      G D \sqrt{\frac{d \log (1/\tau)}{n}}
  \, \right).\]
\end{theorem}
See \cref{app:nonsmooth} for proofs of both results above.

We can show that the excess risk bound in \cref{thm:nonsmooth_noncentral}
is minimax-optimal (up to logarithmic factors) under the noncentral moment assumption, as a matching lower bound can be established.

\begin{theorem}\label{thm:lower-bound_noncentral}
  For \(d \ge 140\) and \(n \ge 62500\), there exist a closed bounded set \(\mathcal{W} \subset \mathbb{R}^d\) with diameter at most \(D\), and a distribution \(p^*\) over functions \(f: \mathcal{W} \rightarrow \mathbb{R}\) that satisfy the following: Let \(\overline{f} = \mathbf{E}_{f\sim p^*}[f]\). We have that for every \(w \in \mathcal{W}\) and unit vector \(v\) that \(\mathbf{E}_{f \sim p^*}[(v \cdot \nabla f(w))^2] \le G^2\). Both \(f\) (almost surely) and \(\overline{f}\) are convex, Lipschitz and smooth. The output \(\hat w\) of any algorithm with access to an \(\epsilon\)-corrupted set of functions \(f_1, \dots, f_n\) sampled from \(p^*\) satisfies the following with probability at least \(1/2\),
  \begin{equation}
    \overline f(\hat w) - \overline f^* = \Omega\left(D G \sqrt{\epsilon} + D G \sqrt{\frac{d}{n}} \right).
  \end{equation}
\end{theorem}

The proof is essentially the same as that of \cref{thm:lower-bound-robust-stochastic-optimization}, since the same hard instances that are used to establish \cref{thm:lower-bound-robust-stochastic-optimization} (c.f. \cref{app:lower-bound}) can be reused to establish \cref{thm:lower-bound_noncentral}.

\subsection{Handling Unknown Covariance Parameter \texorpdfstring{\(\sigma\)}{sigma}}\label{sec:unknown-sigma}
In \cref{alg:robust_net}, \(\sigma\) primarily affects the fineness of the net through \(\xi = \sigma \sqrt{\epsilon} / \bar \beta\). When \(\sigma\) is unknown, we can adapt our algorithm with a preprocessing step to estimate \(\sigma\): (1) Run iterative filtering (\cref{alg:filtering}) to obtain a lower bound \(\hat{\sigma}\) of \(\sigma\), and (2) Run \cref{alg:robust_net} with the modified fineness parameter \(\xi = \hat{\sigma} \sqrt{\epsilon} / \bar \beta\). This adaptation preserves the optimal excess risk guarantees for the known \(\sigma\) case. Full details are provided in \cref{app:unknown-sigma}.

\section{Projected Gradient Descent with Robust Gradient Estimator}\label{sec:robust_gd}

\cref{alg:robust_net} uses a net-based approach to estimate gradients robustly. A more naïve approach is to directly estimate gradients at arbitrary points using a robust gradient estimator. We will show that the simple projected gradient descent algorithm can achieve the same optimal rate as \cref{alg:robust_net} under stronger assumptions.
Even so, our new assumptions are still slightly weaker than those used in SEVER~\citet{diakonikolas2019sever}.
Concretely, following assumptions on the distribution over functions are assumed.
\begin{assumption}\label{assump:smooth-bounded-var-sever}
  Let \(p^*\) be a distribution over functions \(f: \mathcal{W} \rightarrow \mathbb{R}\) with \(\overline{f} = \mathbf{E}_{f\sim p^*}[f]\) so that:
  \begin{enumerate}
    \item \(f - \overline{f}\) is \(L\)-Lipschitz and \(\beta\)-smooth almost surely, where\footnote{
      Without loss of generality, see \cref{sec:discussion_assump}.
    }
     \(L \ge \sigma\).
    \item \(\overline f\) is \(\bar \beta\)-smooth or \(\bar L\)-Lipschitz. 
  \end{enumerate}
\end{assumption}
We use bars on the constants \(\bar\beta\) and $\bar L$  to emphasize that they reflect properties of \(\overline f\).

\cref{alg:robust_gd} follows the ``many-good-sets'' assumption. We are able to robustly estimate the gradient of the population risk \(\overline{f}\) at any point \(w\) with high probability, at the cost of requiring additional almost-sure assumptions on \(f-\overline{f}\) compared to \cref{alg:robust_net}.

\begin{algorithm}[H]
  \caption{Projected Gradient Descent with Robust Gradient Estimator}\label{alg:robust_gd}
  \begin{algorithmic}[1]
    \State \textbf{Input:} \(\epsilon\)-corrupted set of functions \(f_1, \ldots, f_n\), stepsize parameters \(\{\eta_t\}_{t \in [T]}\), robust gradient estimator \(\texttt{RobustEstimator}(w)\).
    \State Initialize \(w_0 \in \mathcal{W}\) and \(t = 1\).
    \For{\(t \in [T]\)}
    \State Apply robust gradient estimator to get \(\tilde g_t = \texttt{RobustEstimator}(w_{t-1})\).
    \State \(w_t \gets \text{Proj}_{\mathcal{W}}(w_{t-1} - \eta_t \tilde g_t)\).
    \EndFor
    \State \textbf{Output:} \(\hat w_T = \frac{1}{T} \sum_{t=1}^T w_t\).
  \end{algorithmic}
\end{algorithm}

\cref{alg:robust_gd} achieves the same optimal excess risk bounds as in \cref{thm:robust_net}.
\begin{theorem}\label{thm:robust-gd}
  Suppose that \cref{assump:conv_loss}, \cref{assump:cov},  and \cref{assump:smooth-bounded-var-sever} hold. 
  There are choices of stepsizes \(\{\eta_t\}_{t=1}^T\) and \(T\) such that, with probability at least \(1-\tau\), we have
  \[\overline f(\hat w_T) - \min_{w \in \mathcal{W}} \overline f(w) = \tilde O \left(
      \sigma D \sqrt{\epsilon} + \sigma D \sqrt{\frac{d \log (1/\tau)}{n}} \,
  \right).\]
  As a consequence, the algorithm achieves excess risk of \(O(D \sigma\sqrt{\epsilon})\) with high probability whenever \(n = \tilde \Omega(d / \epsilon)\).
  The expected excess risk is bounded by \(\tilde O \left(
      \sigma D \sqrt{\epsilon} + \sigma D \sqrt{d/n}
  \right)\).

\end{theorem}

The proof is based on the net argument, similar to that of Lemma C.5 in \citet{li2024robust}.
The high level idea is as follows: For simplicity, we say \(w\) is ``good'' if there exists a good set of functions at \(w\).
We will show that with high probability, there exists a good set for all \(w\) (cf. \cref{assump:many-good-sets}), so that we can robustly estimate the gradient at all \(w\). To show this, we employ a net argument, based on the claim that if \(w\) is ``good'', then all points in a small neighborhood of \(w\) are also ``good.''
By the union bound, with high probability, all points in the net are ``good''. Then it follows that all \(w\) are ``good''.
The full proof can be found in \cref{app:proof_alg2}.

\section{Conclusion and Future Work}\label{sec:conclusion}
In this work, we have advanced robust stochastic convex optimization under the \(\epsilon\)-contamination model. While the prior state of the art SEVER \citep{diakonikolas2019sever} focused on finding approximate critical points under stringent assumptions, we have developed algorithms that directly tackle population risk minimization, obtaining the optimal excess risk under more practical assumptions.
Our first algorithm (\cref{alg:robust_net}) achieves the minimax-optimal excess risk by leveraging our relaxed ``dense-good-sets'' assumption and estimating gradients only at points in a net of the domain, relaxing the stringent distributional conditions as required in SEVER.\@
Our second algorithm (\cref{alg:robust_gd}) provides a simple projected gradient descent approach that achieves the same optimal excess risk, making use of the “many-good-sets” assumption briefly noted in \citep{diakonikolas2019sever}.
Both of our algorithms significantly reduce sample complexity compared to SEVER.

For future work, it would be interesting to explore following directions:
(1) Our excess risk bound is tight up to logarithmic factors. Can we improve the bound to remove the logarithmic factors? (2) Our lower bound is with constant probability. Is it possible to derive a lower bound that the includes \(\log(1/\tau)\) term with probability \(\tau\)?
(3) Robustness has been shown to be closely related to differential privacy \citep{hopkins2023robustness}.
Can we design optimization algorithms that are both robust and differentially private?

\bibliography{ref}

\appendix

\section{Fixing SEVER's Sample Complexity Result}\label{app:sever-fix}
In this section, we fix SEVER's sample complexity result in their Proposition B.5.
Their proof is incorrect due to the error while applying Hoeffding's inequality, which leads to a wrong sample complexity bound (\(n = \tilde \Omega(dL^2/(\epsilon \sigma^2))\)).

We will provide a correct, more rigorous proof for their result.
The correct bound is worse than the one claimed in their paper, as we show below.

\begin{lemma}[{Fixed from Proposition B.5 in~\cite{diakonikolas2019sever}}]\label{lem:sever-cond-uniform}
  Let \(p^*\) be a distribution over functions \(f: \mathcal{W} \rightarrow \mathbb{R}\) with \(\overline{f} = \mathbf{E}_{f\sim p^*}[f]\) so that \(f - \overline{f}\) is \(L\)-Lipschitz and \(\beta\)-smooth almost surely.  Assume further that for every \(w \in \mathcal{W}\) and unit vector \(v\) that \(\mathbf{E}_{f \sim p^*}[(v \cdot (\nabla f(w) - \nabla \overline{f}(w)))^2] \le \sigma^2\). If
  \(
    n = \tilde \Omega\left(
      \frac{dL^2}{\epsilon \sigma^2} + \frac{dL^4}{\sigma^4}
    \right),
  \)
  then with high probability, any an \(\epsilon\)-corrupted set of functions \(f_1, \dots, f_n\) (with respect to \(p^{\ast}\)) satisfy \cref{assump:uniform-good-set}.
\end{lemma}

\begin{proof}
  For any set of functions \(A = \{f_i\}_{i \in A}\) and functional \(G\), we denote \(\E_{i \in A}[G(f_i)]\) as the empirical average of \(G(f_i)\) over \(A\), i.e., \(\frac{1}{|A|} \sum_{i \in A} G(f_i)\).
  Let \(T\) denote the original uncontaminated samples, and let \(S\) denote the \(\epsilon\)-contaminated samples of \(T\).
  Let \(\goodset \subset S\) be the set of uncorrupted functions \(f_i\).
  It is then the case that \(\goodset \subset T\) and that \(|\goodset|\geq (1-\epsilon)n\).

  Let \(\mathcal{N}\) denote a \(\epsilon\)-net of the unit sphere with size \(|\mathcal{N}| = \exp(\tilde O(d))\). Let \(\mathcal{C}\) be a \(\gamma\)-net of \(\mathcal{W}\) with size \(|\mathcal{C}| = \exp(\tilde O(d))\).
  It suffices (choosing small enough \(\epsilon\)) to show that for any \(v \in \mathcal{N}\), with high probability, the following holds: for every \(w \in \mathcal{W}\),
  \begin{equation}
    \begin{aligned}
      \E_{i \in \goodset}[(v\cdot (\nabla f_i(w)-\overline f(w)))] & \leq O(\sigma \sqrt{\epsilon}), \\
      \E_{i \in \goodset}[(v\cdot (\nabla f_i(w)-\overline f(w)))^2] & \leq O(\sigma^2).
    \end{aligned}
  \end{equation}
  We will fix \(w\) for the moment and show that the above holds for all \(w\) in a small neighborhood of \(w\) using a net argument.
  By the bounded covariance assumption,
  \[
    \E_{f \sim p^\ast}[(v\cdot (\nabla f(w)-\overline f(w)))^2] \leq \sigma^2.
  \]
  Observe that the term inside the expectation is bounded by \(L^2\).
  By Hoeffding's Inequality, with probability at least \(1-\exp(-2n (L^4 \log (2(|\mathcal{N}| \cdot |\mathcal{C}|) / n) \tau) /L^4)\), that is, \(1 - \tau / (2(|\mathcal{N}| \cdot |\mathcal{C}|))
  \), we have
  \begin{equation}\label{eq:sever_v}
    \E_{i \in T}[(v\cdot (\nabla f(w)-\overline f(w)))^2] \leq \sigma^2 + L^2 \sqrt{\frac{1}{n}\log ((|\mathcal{N}| \cdot |\mathcal{C}|) / \tau)} = O(\sigma^2).
  \end{equation}
  Now, since \(\goodset \subset T\) and they differ by at most \(\epsilon n\) samples, we know that we have
  \begin{equation}\label{eq:goodset_g2}
    \E_{i \in \goodset}[(v\cdot (\nabla f_i(w)-\overline f(w)))^2] \leq O(\sigma^2).
  \end{equation}

  For the other part, we start by observing that,
  \[
    \E_{f \sim p^\ast}[(v\cdot (\nabla f(w)-\overline f(w)))] = 0.
  \]
  By Chernoff (Hoeffding) bound, with probability at least
  \(1 - \tau / (2(|\mathcal{N}| \cdot |\mathcal{C}|))\), we have that
  \begin{equation}
    \E_{i \in T}[(v\cdot (\nabla f_i(w)-\overline f(w)))] \leq O(L \sqrt{\frac{1}{n}\log ((|\mathcal{N}| \cdot |\mathcal{C}|) / \tau)}) = O(\sigma \sqrt{\epsilon}).
  \end{equation}
  For any subset \(S_1 \subset T\) with \(|S_1| \le \epsilon n \), by Cauchy-Schwarz inequality, we have
  \begin{equation}
    \begin{aligned}
      \frac{1}{n} \sum_{i\in S_1} v\cdot (\nabla f_i(w)-\overline f(w))
      & \leq \frac{1}{n} \sqrt{ |S_1| |} \left(\sum_{i \in T} (v\cdot (\nabla f_i(w)-\overline f(w)))^2\right)^{1/2} \\
      & \leq \frac{1}{n} \sqrt{\epsilon n} \sqrt{O(\sigma^2 n)} = O(\sigma \sqrt{\epsilon}),
    \end{aligned}
  \end{equation}
  where we used~\eqref{eq:sever_v}.
  Therefore, removing \(\epsilon\)-fraction of these samples cannot change this value by more than \(\sigma \sqrt{\epsilon}\).
  Since \(\goodset \subset T\) and they differ by at most \(\epsilon n\) samples, we know that the above bound holds for \(\goodset\) as well, that is
  \begin{equation}\label{eq:goodset_g1}
    \E_{i \in \goodset}[(v\cdot (\nabla f_i(w)-\overline f(w)))] \leq O(\sigma \sqrt{\epsilon}).
  \end{equation}
  We now proceed with a net argument and show~\eqref{eq:goodset_g2} and~\eqref{eq:goodset_g1} hold for all \(w\) with high probability. Suppose they hold for some \(w\), then by \(L\)-Lipschitzness and \(\beta\)-smoothness of \(f(w) - \overline f(w)\), we have that~\eqref{eq:goodset_g1} holds for all \(w'\) within distance \(\sigma \sqrt{\epsilon}/\beta\) from \(w\), and~\eqref{eq:goodset_g2} holds for all \(w'\) within distance \(\sigma^2/(2L \beta)\) from \(w\), where the first statement follows directly from the Lipschitzness and the second statement is due to the following calculation:
  \[
    \begin{aligned}
      & \left| \left\{v \cdot \left(\nabla f_i(w)-\overline f(w)\right)\right\}^2 - \left\{v\cdot \left(\nabla f_i(w')-\overline f(w')\right)\right\}^2 \right| \\
      & \phantom{\qquad} = \left\{v \cdot \left(\nabla f_i(w)-\overline f(w)\right) + v\cdot \left(\nabla f_i(w')-\overline f(w')\right)\right\} \cdot \left\{v\cdot \left(\nabla f_i(w)-\overline f(w)\right)- v \cdot \left(\nabla f_i(w')-\overline f(w')\right)\right\} \\
      & \phantom{\qquad} \le 2 L \cdot \beta \| w - w'|,
    \end{aligned}
  \]
  for any unit vector \(v\).
  It suffices to choose the fineness of the net \(\mathcal{C}\) as \(\gamma = \min(\sigma \sqrt{\epsilon}/\beta, \sigma^2/(2L \beta))\) so that for any \(w\), there exists a point in the net within distance \(\gamma\) from \(w\).

  Applying a union bound over net \(\mathcal{C}\) and \(\mathcal{N}\),
  we have that~\eqref{eq:goodset_g2} and~\eqref{eq:goodset_g1} hold for all \(w\) with probability at least \(1 - \tau\).
\end{proof}

\section{Lower Bound for Robust Stochastic Optimization}\label{app:lower-bound}

In this section, we demonstrate a matching lower bound for robust stochastic optimization under \(\epsilon\)-strong contamination with bounded covariance, showing that our algorithm achieves the minimax-optimal excess risk rate (up to logarithmic factors).
We will prove \cref{thm:lower-bound-robust-stochastic-optimization}.
\begin{remark}
  Our construction of hard instances meets the superset of the assumptions of our algorithms and SEVER, i.e.,
  Lipschitzness and smoothness of the individual functions (consequently the same holds for the population risk), and bounded covariance of the gradients.
  The lower bound consists of two terms. The first term is due to corruption and the second term is necessary even without corruption.
  We will prove these two terms separately.
\end{remark}

\subsection{Lower Bound: Term due to Corruption}
We will leverage the following proposition that characterizes the information-theoretic limit of robust estimation.
\begin{proposition}[\citet{diakonikolas2023algorithmic}]\label{prop:robust-lower-bound}
  Let \(X\) and \(Y\) be distributions with \(d_{\text{TV}}(X, Y) \le 2\epsilon\) for some \(0 < \epsilon < 1\). A distribution \(\mathcal{D}\) is taken to be either \(X\) or \(Y\). Then an algorithm, given any number of samples from \(\mathcal{D}\) under \(\epsilon\)-contamination, cannot distinguish between the cases \(\mathcal{D} = X\) and \(\mathcal{D} = Y\) with probability greater than \(1/2\).
\end{proposition}

Consider a random variable \(X_1\) that takes value 0 with probability \(1-\epsilon\) and takes value \(\pm \sigma / \sqrt{\epsilon}\) with probability \(\epsilon/2\) each. That is,
\begin{equation}
  X_1 =
  \begin{cases}
    0 & \text{with probability } 1-\epsilon \\
    \sigma / \sqrt{\epsilon} & \text{with probability } \epsilon/2 \\
    -\sigma / \sqrt{\epsilon} & \text{with probability } \epsilon/2
  \end{cases}.
\end{equation}
The variance of \(X_1\) is \(\sigma^2\). Now, consider \(X_1'\) that takes value 0 with probability \(1-\epsilon\) and takes value \(\sigma / \sqrt{\epsilon}\) with probability \(\epsilon\).
The mean of \(X_1'\) is \(\sigma \sqrt{\epsilon}\), and
the variance of \(X_1'\) is

\begin{equation}
  \begin{aligned}
    \text{Var}(X_1') &= \mathbb{E}[X_1'^2] - \left(\mathbb{E}[X_1']\right)^2 \\
    &= \epsilon \cdot \frac{\sigma^2} \epsilon + 0 - (\sigma \sqrt{\epsilon})^2 \\
    &= \sigma^2 - \sigma^2 \epsilon < \sigma^2.
  \end{aligned}
\end{equation}

Let \(\mathcal{D}_1\) and \(\mathcal{D}_1'\) denote the probability distributions of \(X_1\) and \(X_1'\) respectively.

Consider the following robust optimization instance.
Let \(\{w \mid |w| \le D\}\) denote the feasible set. Define the loss function as \(f_x(w) = - w \cdot x\). We know that \(\nabla f_x(w) = -x\), so that both \(f\) and \(\overline f\) are Lipschitz and smooth. Let \(\overline f(w, \mathcal{D}) = \mathbb{E}_{X \sim \mathcal{D}}[f_X(w)]\) denote the population risk for a given distribution \(\mathcal{D}\). It is easy to verify the bounded covariance assumption on gradients holds for both \(\overline f(w, \mathcal{D})\) and \(\overline f(w, \mathcal{D'})\) with \(\sigma^2\) as the variance.

Expanding the expectation, we have that
\begin{equation}
  \begin{aligned}
    -\overline f(w, \mathcal{D}_1) & = (1-\epsilon) w \cdot 0 + \frac{\epsilon}{2} w \cdot \frac{\sigma}{\sqrt{\epsilon}} + \frac{\epsilon}{2} w \cdot -\frac{\sigma}{\sqrt{\epsilon}} = 0, \\
    -\overline f(w, \mathcal{D}_1') & = (1-\epsilon) w \cdot 0 + \epsilon w \cdot \frac{\sigma}{\sqrt{\epsilon}} = w \cdot \sigma \sqrt{\epsilon}.
  \end{aligned}
\end{equation}
So we have
\begin{equation}
  \min_{w} \overline f(w, \mathcal{D}_1) = 0 \quad \text{and} \quad \min_{w} \overline f(w, \mathcal{D}_1') = -D \sigma \sqrt{\epsilon}.
\end{equation}
Therefore,
\begin{equation}
  \min_{w} \overline f(w, \mathcal{D}_1) - \min_{w} \overline f(w, \mathcal{D}_1') = D \cdot \sigma \sqrt{\epsilon}.
\end{equation}
The total variation distance between \(\mathcal{D}_1\) and \(\mathcal{D}_1'\) is \(\epsilon\).
Therefore, by \cref{prop:robust-lower-bound}, given \(\epsilon\)-corrupted samples, no algorithm can distinguish when these samples are generated from \(\mathcal{D}_1\) or \(\mathcal{D}_1'\) with probability greater than \(1/2\).
If an algorithm could optimize the population risk within \(D \cdot \sigma \sqrt{\epsilon}\) with probability greater than \(1/2\), then it could use the output to distinguish between \(\mathcal{D}_1\) and \(\mathcal{D}_1'\), which is a contradiction.

\subsection{Lower Bound: Term due to Stochastic Optimization}
The proof below is essentially the same as the proof in \citet{lowy2023private}.
\begin{lemma}[{\citet{lowy2023private}[\textbf{Theorem} 36, part 3 for \(k=2\) and \(\gamma = \sigma^2\)]}]
  There exists a product distribution (the distribution of product of independent random variables) \(Q_{\nu^*}\) that is supported on \( \{\pm \sigma\}^d\). Let the feasible set be \(\mathcal{W} = B_2^d(0, D)\). Define the loss \(f_x(w) = - \langle w,x \rangle\), \(\overline f(w) := \mathbb{E}_{x \sim Q_{\nu^*}} f_x(w)\). We have that
  \(\sup_{w \in \mathcal{W}} \mathbb{E}_{x \sim Q_{\nu^*}} | \langle \nabla f_x(w) - \overline f(w), e_j \rangle |^2 \leq \sigma^2\),  for all \(j \in [d]\), where \(e_j\) denotes the \(j\)-th standard basis vector in \(\mathbb{R}^d\).

  Let \(X \sim Q_\nu^n\). Any algorithm \(\mathcal{A}\) has the following excess risk lower bound with probability at least \(1/2\),
  \begin{equation}\label{eq:lowy-lower-bound}
    \mathbb{E} \overline f(\mathcal{A}(X)) - \overline f^* = \Omega\left(D\sigma \sqrt{\frac{d}{n}} \right).
  \end{equation}
\end{lemma}

\begin{proof}
We first consider \(\sigma=1\) and \(D = 1\).
Following~\cite{lowy2023private}, let \(f_x(w) = -w^T x\) with
\(\mathcal{W} = B_2^d(0,1)\) and \(\mathcal{X} = \{\pm 1\}^d\)

By the Gilbert-Varshamov bound
and the assumption \(d \geq 40\), there exists a set \(\mathcal{V}
\subseteq \{\pm 1\}^d\) with \(|\mathcal{V}| \geq 2^{d/20}\),
\(d_{\text{Ham}}(\nu, \nu') \geq \frac{d}{8}\) for all \(\nu, \nu'
\in \mathcal{V}, ~\nu \neq \nu'\). For \(\nu \in \mathcal{V}\),
define the product distribution \(Q_\nu = (Q_{\nu_1}, \cdots
Q_{\nu_d})\), where for all \(j \in [d]\),
\[
Q_{\nu_j} =
\begin{cases}
    1 &\mbox{with probability \(\frac{1 + \delta_{\nu_j}}{2}\)}\\
    -1 &\mbox{with probability \(\frac{1 - \delta_{\nu_j}}{2}\)}
\end{cases}
\]
for \(\delta_{\nu_j} \in (0,1)\) to be chosen later. Then
\(\mathbb{E} Q_{\nu_j} := \mu_{\nu_j} = \delta_{\nu_j}\) and for
any \(w \in \mathcal{W}\), \(x \sim Q_{\nu}\), we have
\begin{equation}
\begin{aligned}
    \mathbb{E} |\langle \nabla f_x(w) - \nabla F(w), e_j \rangle|^k
    &= \mathbb{E}|\langle -x + \mathbb{E} x, e_j \rangle|^k \\
    &= \mathbb{E}|x_j - \mu_{\nu_j}|^k \\
    & = \frac{1 + \delta_{\nu_j}}{2}|1 - \delta_{\nu_j}|^k +
    \frac{1-\delta_{\nu_j}}{2}|1 + \delta_{\nu_j}|^k \\
    &\leq 1 - \delta_{\nu_j}^2 \leq 1
\end{aligned}
\end{equation}
for \(\delta_{\nu_j} \in (0,1)\).
Now, let \(p:= \sqrt{d/n}\)
and \(\delta_{\nu_j} := \frac{p \nu_j}{\sqrt{d}}
    \). Note that \(\mathbb{E} Q_\nu := \mu_\nu = \frac{p
\nu}{\sqrt{d}}\) and \(w_\nu := \frac{\mu_\nu}{\|\mu_\nu\|} =
\frac{\nu}{\|\nu\|}\). Also, \(\| \mu_\nu \| =
    p
:= \| \mu\|\) for all \(\nu \in \mathcal{V}\). Now, denoting
\(F_{Q_{\nu}}(w) := \mathbb{E}_{x \sim Q_{\nu}} f_x(w)\), we have
for any \(\nu \in \mathcal{V}\) and
\(w \in \mathcal{W}\) (possibly depending on \(X \sim Q_{\nu}^n\)) that
\begin{equation}
\begin{aligned}
    F_{Q_{\nu}}(w) - \min_{w' \in \mathcal{W}} F_{Q_{\nu}}(w') &=
    \left \langle \frac{\mu_\nu}{\|\mu\|}, \mu_\nu \right \rangle -
    \left\langle w, \mu_\nu \right\rangle \\
    &=  \|\mu\| - \langle w, \mu_{\nu} \rangle \\
    &= \|\mu\|[1 - \langle w, w_\nu \rangle] \\
    &\geq \frac{1}{2} \|\mu\| \| w - w_\nu \|^2 ,
\end{aligned}
\end{equation}
since \(\|w\|, \|w_\nu\| \leq 1\). Further, denoting \(\hat{w} :=
\argmin_{w \in \mathcal{W}}\|w_\nu - w\|\), we have \(\|\hat{w} -
w_\nu \|^2 \leq 4\|w_\nu - w\|^2\) for all \(\nu \in \mathcal{V}\)
(via Young's inequality). Hence, we have
\begin{equation}\label{eq:reduction}
F_{Q_{\nu}}(w) - \min_{w' \in \mathcal{W}} F_{Q_{\nu}}(w') \geq
\frac{\| \mu \|}{8} \|\hat{w} - w_\nu\|^2.
\end{equation}

For all \(\nu \neq \nu'\), we have \(\|w_\nu - w_{\nu'}\|^2 \geq
\frac{\|\nu - \nu'\|^2}{\|\nu\|^2} \geq 2 \cdot \tfrac{1}{2}\)
since \(d_{\text{Ham}}(\nu, \nu') \geq \frac{d}{2}\) and \(\nu \in
\{\pm 1\}^d\) implies \(\|\nu - \nu'\|^2 \geq \frac{d}{2}\) and
\(\|\nu\|^2 = d\).

We can compute that for any \(j \in [d]\) and \(\nu, \nu' \in \mathcal{V}\),
\begin{equation}
\begin{aligned}
    D_{KL}(Q_{\nu_j} || Q_{\nu'_j}) &\leq \frac{1 +
    \frac{p}{\sqrt{d}}}{2}\left[\log\left(\frac{\sqrt{d} +
    p}{\sqrt{d}}\right) + \log\left(\frac{\sqrt{d}}{\sqrt{d} -
    p}\right)\right]\\
    & \leq \log\left(\frac{1 + \frac{p}{\sqrt{d}}}{1 -
    \frac{p}{\sqrt{d}}}\right) \\
    &\leq \frac{3p}{\sqrt{d}},
\end{aligned}
\end{equation}
for our choice of \(p\), provided \(\frac{p}{\sqrt{d}} =
\frac{1}{\sqrt{n}} \in (0, \frac{1}{2})\), which holds if
\(n > 4\). Hence, by the chain rule for KL-divergence,
\[
D_{KL}(Q_\nu || Q_{\nu'}) \leq 3p\sqrt{d}
= \frac{3d}{\sqrt{n}},
\]
for all \(\nu, \nu' \in \mathcal{V}\).

Now consider the loss function \(\phi(\|\hat w  - w\|)\) where
\(\phi(x) := \mathbbm{1}\{x \ge \tfrac{1}{2}\}\). We have
\(\mathbb{E}[\phi(\|\hat w  - w\|)] = P\left(\|\hat w  - w\| \ge
\tfrac{1}{2}\right)\) for all \(w \in \mathcal{W}\).
Applying Fano's method\footnote{
  Although the original Fano's method is for deterministic algorithms, by data processing inequality, the same result holds for randomized algorithms.
}, we have
\[
\inf_{\hat w} \max_{\nu \in \mathcal{V}} P\left(\| \hat w - w_\nu
\| \ge \tfrac{1}{2}\right) \ge \phi(\tfrac{1}{2})\left(1 -
\frac{3p\sqrt{d} + \log(2)}{d \log(2)/20} \right) = 1 -
\frac{60}{\sqrt{n}\log(2) } - \frac{20}{d},
\]
which is greater than \(1/2\) for
\(d \ge 140\) and \(n \ge 62500\). Then there exists \(\nu^* \in
\mathcal{V}\) such that with probability at least \(1/2\), let
\(\hat w\) denote the output of any deterministic algorithm, we have
\(
\|\hat w - w_{\nu^*}\| \ge \tfrac{1}{2}.
\)

Using \eqref{eq:reduction} and plugging in \(\|\mu\| =
\sqrt{d/n}\), we have with probability at least \(1/2\),
\begin{equation}
F_{Q_{\nu^*}}(w) - \min_{w' \in \mathcal{W}} F_{Q_{\nu^*}}(w') =
\Omega\left(\sqrt{d/n}\right).
\end{equation}
For the general case, we can scale \(\mathcal{X}\) and
\(\mathcal{V}\) by \(\sigma\), and scale \(\mathcal{W}\) by \(D\).
This gives us a lower bound of \(\Omega(D\sigma\sqrt{d/n})\) with
probability at least \(1/2\).

We now verify that the constructed hard instance satisfies the
assumptions used in our algorithms. We return to the notation we
used in the main text, and write
\(f_x(w) = - \langle w,x \rangle\), \(\overline f(w) :=
\mathbb{E}_{x \sim Q_{\nu^*}} f_x(w)\).
We know that \(\nabla_w f_x(w) = -x\), so that both \(f\) and
\(\overline f\) are Lipschitz and smooth.
It remains to show that this hard instance satisfies the bounded
covariance assumption. For any unit vector \(u \in \mathbb{R}^d\),
write \(u = \sum_{j=1}^d u_j e_j\). Then
\begin{equation}
\begin{aligned}
    \mathbb{E}[(\langle u, \nabla f(w) - \nabla \overline{f}(w) \rangle )^2]
    & = \mathbb{E}\left(\sum_{j=1}^d u_j \langle e_j, \nabla f(w) -
    \nabla \overline{f}(w) \rangle \right)^2 \\
    & \stackrel{(*)}{=} \mathbb{E}\left[\sum_{j=1}^d u_j^2 (\langle
    e_j, \nabla f(w) - \nabla \overline{f}(w) \rangle )^2\right] \\
    & = \sum_{j=1}^d u_j^2 \mathbb{E}[(\langle e_j, \nabla f(w) -
    \nabla \overline{f}(w) \rangle )^2] \\
    & \leq \sum_{j=1}^d u_j^2 \sigma^2 = \sigma^2,
\end{aligned}
\end{equation}
where in \((*)\) we use the fact that \(Q_{\nu}\) is a product
distribution and thus cross terms vanish in the expectation.
\end{proof}
Therefore, the lower bound in~\eqref{eq:lowy-lower-bound} is also a lower bound for our problem. Combining this with the lower bound term due to corruption, we have the lower bound desired.

\section{Proof of Theorem~\ref{thm:robust_net}}\label{app:proof_alg1}
We start with the robust estimation result from \citet{diakonikolas2020outlier}, then proceed with the proof of  \cref{thm:robust_net}. 
\begin{lemma}[{\cite[Proposition 1.5]{diakonikolas2020outlier}}]
\label{lem:mean_estimation_bound}
  Let \(S\) be an \(\epsilon\)-corrupted set of \(n\) samples from a distribution in \(\R^d\)
  with mean \(\mu\) and covariance \(\Sigma\) such that \(\Sigma \preceq \sigma^2 I\).
  Let \(\epsilon' =  \Theta(\log(1/\tau)/n + \epsilon) \leq c\) be given, for a constant \(c > 0\).
  Then any stability-based algorithm (e.g. \cref{alg:filtering}) on input \(S\) and \(\epsilon'\), efficiently computes \(\widehat{\mu}\) such that with probability at least
  \(1 - \tau\), we have
  \begin{equation}
    \|\widehat{\mu} - \mu\| = O(\sigma \cdot \delta(\tau)), \; \text{where} \; \delta(\tau) = \sqrt{\epsilon} + \sqrt{d/n} + \sqrt{\log(1/\tau)/n}.
  \end{equation}
\end{lemma}

\begin{proof}[Proof of \cref{thm:robust_net}]\label{proof:robust-net}
  1. Bound the bias of the gradient estimator at \(w\).
  For given \(w\), let \(w' = \argmin_{z \in \mathcal{C}} \| z - w \|\). Applying \cref{lem:mean_estimation_bound} to samples \(\nabla f_1(w'), 
  \nabla f_2(w'), \dots, \nabla f_n(w')\), we have that with probability at least \(1 - \tau'\), the robust gradient estimator \(\tilde g(w')\) satisfies
  \[
    \|\tilde g(w') - \nabla \overline{f}(w')\| = \sigma \cdot \tilde O\left(\sqrt{\epsilon} + \sqrt{d/n} + \sqrt{\log(1/\tau')/n}\right).
  \]
  We have \( \| w - w' \| \le \sigma \sqrt{\epsilon} / \bar \beta \) by definition of the net. By \(\bar \beta\)-smoothness of the population risk \(\overline{f}\), we have
  \begin{equation}\label{eq:smooth_error}
    \| \nabla \overline{f}(w) - \nabla \overline{f}(w') \| \le \bar \beta \| w - w' \| \le \sigma \sqrt{\epsilon}.
  \end{equation}
  Combining the two bounds, we have
  \begin{equation}\label{eq:est_bound_analysis}
    \| \tilde g(w') - \nabla \overline{f}(w) \| = \sigma \cdot \tilde O\left(\sqrt{\epsilon} + \sqrt{d/n} + \sqrt{\log(1/\tau')/n} \right).
  \end{equation}

  2. Apply the union bound over all points in the net \(\mathcal{C}\).
  By union bound, setting \( \tau' = \tau / |\mathcal{C}| \), we have that with probability at least \(1 - \tau\),~\eqref{eq:est_bound_analysis} simultaneously holds for all \(w' \in \mathcal{C}\).
  Recall \(|\mathcal{C}| = O\big(D \sqrt{d}/\xi \big)^d\). We have \(\log |\mathcal{C}| = \tilde O(d)\).
  It follows that, with probability at least \(1 - \tau\), simultaneously for all \(w \in \mathcal{W}\), let \(w' = \argmin_{z \in \mathcal{C}} \| z - w \|\), we have
  \begin{equation}
    \| \tilde g(w') - \nabla \overline{f}(w) \| = \sigma \cdot \tilde O\left(\sqrt{\epsilon} + \sqrt{d/n} + \sqrt{d \log(1/\tau)/n} \right).
  \end{equation}
  Therefore, with probability at least \(1 - \tau\), the bias of the gradient estimator at \(w\) is bounded by the above expression, simultaneously for all \(w \in \mathcal{W}\).

  3. Apply the projected biased gradient descent analysis.
  By \cref{lem:biased_gd_smooth_convex}, choosing a constant step size \(\eta = 1/\bar \beta\), the excess risk of the algorithm is bounded by
  \begin{equation}
    \overline f
    (\hat w_T) - \min_{w \in \mathcal{W}} \overline f(w) = \tilde O \left(
      \frac{\bar \beta D^2}{T} +
      D \cdot \left(\sigma \sqrt{\epsilon} + \sigma \sqrt{\frac{d \log (1/\tau)}{n}}\right)
    \right).
  \end{equation}
  Choosing \(T = \tilde \Omega\left(\frac{\bar \beta D}{\sigma \sqrt{\epsilon} + \sigma \sqrt{\frac{d \log (1/\tau)}{n}}}\right)\) gives the optimal rate.
\end{proof}

\section{Analysis of Convolutional Smoothing for Nonsmooth but Lipschitz Population Risk}\label{app:nonsmooth}
Before proving \cref{thm:nonsmooth}, we need some properties of the convolutional smoothing.

\begin{lemma}[\citep{yousefian2012stochastic}]\label{lem:conv_smooth_properties}
  Suppose \(\{f(w)\}\) is convex and \(L\)-Lipschitz over \(\mathcal{W}+B_2(0, s)\), where \(B_2(0, s)\) is the \(d\)-dimensional \(L^2\) ball of radius \(s\) centered at the origin. For \(w \in \mathcal{W}\), the convolutional smoother with radius \(s\), \(\tilde f_s(w) := \E_{u \sim \mathcal{U}_s} [f(w+u)]\), where \(\mathcal{U}_s\) is the uniform distribution over \(B_2(0, s)\), has the following properties:
      \begin{enumerate}
      \item \(f(w) \leq \tilde f_s(w) \leq f(w)+L s\);
      \item \(\tilde f_s(w)\) is convex and \(L\)-Lipschitz;
      \item \(\tilde f_s(w)\) is \(\frac{L \sqrt{d}}{s}\)-smooth.
      \end{enumerate}
  \end{lemma}

\medskip
\begin{proof}[Proof of \cref{thm:nonsmooth}]
  Let \(\overline f_s(w) = \E_{u \sim \mathcal{U}_s} [\overline f(w+u)]\) be the smoothed population risk.
\begin{enumerate}
  \item By properties of convolutional smoothing (part 3 of \cref{lem:conv_smooth_properties}), we know that since \(\overline f\) is \(\bar L\)-Lipschitz, \(\overline f_s\) is \((\bar L\sqrt{d}/s)\)-smooth.
  
  As \(f_1, f_2, \ldots, f_n\) are \(\epsilon\)-corrupted samples from \(p^*\), we know \(f_1(\cdot+u_1), f_2(\cdot+u_2), \ldots, f_n(\cdot+u_n)\) are \(\epsilon\)-corrupted samples from the product distribution of \(p^*\) and \(\mathcal{U}_s\).
  Below, we show that in expectation, perturbed gradient (clean) samples \(\nabla f(w+u)\) is equal to the
  smoothed gradient \(\nabla \overline f_s(w)\).
  
  By the law of total expectation, we have
  \[
  \E_{f \sim p^*, \, u \sim \mathcal{U}_s}\big[\nabla f(w+u)\big]
  = \E_{u \sim \mathcal{U}_s}\Big[ \E_{f \sim p^*}\big[\nabla f(w+u)\big]\Big],
  \]
  and using the regularity condition to interchange the gradient with the expectation, we obtain
  \[
    \E_{f \sim p^*, \, u \sim \mathcal{U}_s}\big[\nabla f(w+u)\big] = \E_{u \sim \mathcal{U}_s}\Big[ \nabla \overline{f}(w+u)\Big].
  \]
  Below we drop the distributions and write \(\E_f\), \(\E_u\) for simplicity.
  Since \(\overline f\) is \(\bar L\)-Lipschitz, we can exchange the order of expectation and gradient, that is,
  \begin{equation}
    \E_{u}\Big[ \nabla \overline{f}(w+u)\Big] = \nabla \E_{u}\Big[\overline{f}(w+u)\Big] = \nabla \overline f_s(w),
  \end{equation}
  which shows that \(\E_{f, u}\Big[ \nabla f(w+u)\Big] = \nabla \overline f_s(w)\).
  
 \item Next, we bound the covariance of the perturbed gradient \(\cov_{f,u}(\nabla f(w+u))\).

  Using law of total covariance, that is, \(\cov(X,Y)=\E(\cov(X,Y \mid Z))+\cov(\E(X\mid Z),\E(Y\mid Z))\), conditioned on \(u\), we can write the covariance of the perturbed gradient as
  \begin{equation}  
  \begin{aligned}
    \cov_{f,u}(\nabla f(w+u))
    & = \E_u \left[ \cov_f(\nabla f(w+u)) \right] + \cov_u \left[ \E_f[\nabla f(w+u)] \right] \\
    & =  \underbrace{\E_u \left[ \cov_f(\nabla f(w+u)) \right]}_{\text{Term 1}} + \underbrace{\cov_u \left[\nabla \overline f(w+u) \right]}_{\text{Term 2}},
  \end{aligned}
  \end{equation}
    In the first term, for all \(u\), the covariance inside the expectation is bounded by \(\sigma^2 I\) by assumption. Therefore, we have 
  \(
  (\text{Term 1}) \preceq \mathbb{E}_u[\sigma^2 I] = \sigma^2 I
  \).
  We bound the second term using the boundedness of the gradient. Consider the following fact: for any random vector \(u\), we have \(\E [u u^\top] \le C^2 I\) if \(\|u\| \le C\) almost surely, and consequently, \(\cov(u) \le 4C^2 I\).
  It follows that,
  \(
  (\text{Term 2}) \preceq 4 \bar L^2 I
  \).
  Therefore, the covariance parameter increases from \(\sigma^2\) to \(\sigma^2 + 4 \bar L^2\) due to smoothing.
 
We have verified the conditions to apply the original algorithm, with \(\sigma\) replaced by \(\sqrt{\sigma^2 + 4 \bar L^2}=O(\sigma+\bar{L})\) and \(\bar \beta = \bar L\sqrt{d}/s\). \medskip %
  
\item By applying \cref{thm:robust_net} to the smoothed function \(\overline f_s\), we know there are choices of \(\{\eta_t\}_{t=1}^T\) and \(T\) such that, with probability at least \(1-\tau\), the output \(\hat w_T\) satisfies,
    \[
    \overline f_s(\hat w_T) - \min_{w \in \mathcal{W}} \overline f_s(w) = \tilde O \left(
      (\sigma + \bar L) D \sqrt{\epsilon} +
      (\sigma + \bar L) D \sqrt{\frac{d \log (1/\tau)}{n}}
  \, \right).
  \]
  By properties of convolutional smoothing, we have
  \[
    \overline f(\hat w_T) \le \overline f_s(\hat w_T)
  \quad \text{and} \quad
    \min_{w \in \mathcal{W}} \overline f_s(w) \le \min_{w \in \mathcal{W}} \overline f(w) + \bar Ls.
  \]
  It follows that, choosing \(s = \tilde O\left(D(\sigma / \bar L + 1)(\sqrt{\epsilon} + \sqrt{d \log (1/\tau) / n})\right)\), we have
  \begin{equation}
  \begin{aligned}
    \overline f(\hat w_T) - \min_{w \in \mathcal{W}} \overline f(w)
    & \le \overline f_s(\hat w_T) - \min_{w \in \mathcal{W}} \overline f_s(w) + \bar Ls \\
    & = \tilde O \left(
      (\sigma + \bar L) D \sqrt{\epsilon} +
      (\sigma + \bar L) D \sqrt{\frac{d \log (1/\tau)}{n}}
  \, \right).
  \end{aligned}
  \end{equation}
\end{enumerate}
\end{proof}
As a corollary, the result under the noncentral moment condition
(\cref{thm:nonsmooth_noncentral}) follows using the lemma below, which the relates noncentral second moment to the mean.
\begin{lemma}\label{lemma:noncentral}
For a \(d\)-dimensional random vector \(u\) that satisfies
\(
\mathbb{E}[uu^\top] \preceq G^2 I_d
\),
we have that \(\|\mathbb{E}[u]\| \le G\).
\end{lemma}

\begin{proof}
Let \(v\) be any fixed unit vector in \(\mathbb{R}^d\). By Jensen's inequality, we have
\begin{equation}
  \left(v^\top \mathbb{E}[u]\right)^2 = \left(\mathbb{E}[v^\top u]\right)^2 \leq \mathbb{E}[(v^\top u)^2] = v^\top \mathbb{E}[uu^\top] v \leq G^2
  \end{equation}

Since this inequality holds for any unit vector \(v\), we can choose \(v = \frac{\mathbb{E}[u]}{\|\mathbb{E}[u]\|}\) when \(\mathbb{E}[u] \neq 0\), which gives \(\|\mathbb{E}[u]\|^2 \leq G^2\), and thus \(\|\mathbb{E}[u]\| \leq G\). If \(\mathbb{E}[u] = 0\), inequality \(\|\mathbb{E}[u]\| \leq G\) holds trivially.
\end{proof}

\begin{proof}[Proof of \cref{thm:nonsmooth_noncentral}]
For any random vector \(u\), we have for any unit vector \(v\),
\[ \mathbb{E}[(v^\top \left(u - \mathbb{E}[u]\right))^2] = \mathbb{E}[(v^\top u)^2] - \|v^\top \mathbb{E}[u]\|^2.
\]
Substituting \(u = \mathbb{E}_f [\nabla f(w)]\), we have \(\E[u] = \nabla \overline{f}(w)\). So for every \(w\), we have \(\mathbf{E}[(v \cdot (\nabla f(w) - \nabla \overline{f}(w)))^2] \le \mathbf{E}[(v \cdot \nabla f(w))^2]\le G^2\) for any unit vector \(v\).
We know that \(\mathbf{E}[\nabla f(w) \nabla f(w)^\top] \preceq G^2 I_d\) is equivalent to \(\mathbf{E}[(v \cdot \nabla f(w))^2]\le G^2\) holding for every unit vector \(v\).
By \cref{lemma:noncentral}, we have \( \| \nabla \overline{f}(w) \| \le G\) for every \(w\).
Therefore, applying \cref{thm:nonsmooth} with \(\sigma = G\) and \(\bar L = G\), we obtain the desired result.
\end{proof}

\section{Analysis of Algorithm~\ref{alg:robust_gd}}\label{app:proof_alg2}
Before proving \cref{thm:robust-gd}, we need some results from robust estimation literature.

\subsection{Results from Robust Mean Estimation}
Recall \cref{def:good-set}. The ``good'' set property is a special case of stability, defined as follows:
\begin{definition}[Stability \citet{diakonikolas2019robust}]\label{def:stability}
  Fix \(0 < \epsilon < 1/2\) and \( \delta \geq \epsilon\).
  A finite set \(S \subset \R^d\) is \((\epsilon,\delta)\)-stable with respect to mean \(\mu \in \R^d\) and \(\sigma^2\)
  if for every \(S' \subseteq S\) with \(|S'| \geq (1 - \epsilon) |S|\),
  the following conditions hold: (i) \(\| \mu_{S'} - \mu\| \leq \sigma \delta\), and
  (ii) \(\|\overline{\Sigma}_{S'}  - \sigma^2 I\| \leq \sigma^2\delta^2/\epsilon\), where \(\mu_{S'} = (1/|S'|) \sum_{x \in S'} x\) and 
  \(\Sigma_{S'} = (1/|S'|) \sum_{x \in S'} (x - \mu)(x - \mu)^{\top}\).
\end{definition}

The following due to \citet{diakonikolas2020outlier} establishes stability of samples
from distributions with bounded covariance.
\begin{lemma}[\citet{diakonikolas2020outlier}]\label{lem:stable_subset}
  Fix any \(0< \tau'<1\).
  Let \(S\) be a multiset of \(n\) i.i.d.\ samples from a distribution on \(\R^d\)
  with mean \(\mu\) and covariance \(\Sigma\) such that \(\Sigma \preceq \sigma^2 I\).
  Let \( \epsilon' =  \Theta(\log(1/ \tau')/n + \epsilon) \leq c\), for a sufficiently small constant \(c>0\).
  Then, with probability at least \(1 - \tau'\), there exists a subset \( S' \subseteq S\) such that
  \( |S'| \geq (1 - \epsilon')n\) and \(S'\) is \( (2\epsilon', \delta')\)-stable with respect to \(\mu\) and \(\sigma^2\), where
  \(\delta' = \delta(\tau')\) depends on \(\tau'\) as \(\delta(\tau') = O(\sqrt{ (d \log d) / n} +  \sqrt{\epsilon} + \sqrt{\log(1/\tau')/n})\).
\end{lemma}

With the stability condition, we can robustly estimate the mean of a distribution with bounded covariance.
\begin{lemma}[Robust Mean Estimation Under Stability \citet{diakonikolas2019robust}]\label{lem:robust_mean_est}
  Let \(T \subset \mathbb{R}^d\) be an \(\epsilon\)-corrupted version of a set \(S\) with the following stability properties: \(S\) contains a subset \(S^{\prime} \subseteq S\) such that \(\left|S^{\prime}\right| \geq(1-\epsilon)|S|\) and \(S^{\prime}\) is \((C \epsilon, \delta)\) stable with respect to \(\mu \in \mathbb{R}^d\) and \(\sigma^2\), for a sufficiently large constant \(C>0\). Then there is a polynomial-time algorithm (e.g. \cref{alg:filtering}), that on input \(\epsilon, T\), computes \(\widehat{\mu}\) such that \(\|\widehat{\mu}-\mu\| = O(\sigma \delta)\).
\end{lemma}

\subsection{Proof of Theorem~\ref{thm:robust-gd}}
As long as the stability condition holds, we can use deterministic stability-based algorithms (e.g. \cref{alg:filtering}) to robustly estimate the mean. Using union bound over the net, it suffices to argue that at a given point \(w\), given the existence of a stable subset of the form \(\{\nabla f_i(w)\}_{i \in \mathcal{I}}\), where \(\mathcal{I}\) denotes the index set of the stable subset at \(w\), such subset is also stable within a small neighborhood of \(w\), that is, \(\{ \nabla f_i(w') \}_{i \in \mathcal{I}}\) is stable for all \(w'\) in a small neighborhood of \(w\). We have the following stability result, which corresponds to ``many-good-sets'' \cref{assump:many-good-sets}.

\begin{lemma}\label{lem:stable_subset_gd}
  Under \cref{assump:cov} and \cref{assump:smooth-bounded-var-sever}, let \(f_1, \ldots, f_n\) denote an \(\epsilon\)-corrupted set of functions sampled from \(p^*\). Let \(\epsilon' =  \Theta(\log(1/\tau)/n + \epsilon) \leq c\) be given, for a constant \(c > 0\). With probability at least \(1 - \tau\), for all \(w \in \mathcal{W}\), there exists index set \(\mathcal{I} \subseteq [n]\) (here \(\mathcal{I}\) depends on the choice of \(w\)) such that \(|\mathcal{I}| \ge (1 - \epsilon')n\) and \(\{ \nabla f_i(w) \}_{i \in \mathcal{I}}\) is \((2\epsilon', \delta(\tau'))\)-stable with respect to \(\nabla \overline{f}(w)\) and \(\sigma^2\),
  where \(\tau' = \tau / \exp(\tilde O(d)) \) and \(\delta(\tau') = \tilde O\left( \sqrt{\epsilon} + \sqrt{d\log(1/\tau)/n}\right)\).
\end{lemma}

\begin{proof}
  We use a net argument to show that the stability condition holds for all \(w\), following similar proof techniques used in \citet{li2024robust}.
  For fixed \(w\), by \cref{lem:stable_subset}, with probability at least \(1 - \tau'\), there exists a subset \(\mathcal{I} \subseteq [n]\) such that \(|\mathcal{I}| \ge (1 - \epsilon')n\) and \(\{ \nabla f_i(w) \}_{i \in \mathcal{I}}\) is \((2\epsilon', \delta')\)-stable where \(\delta' = \delta(\tau')\), with respect to \(\nabla \overline{f}(w)\) and \(\sigma^2\), that is
  \begin{subequations}
    \begin{equation}
      \left \| \frac{1}{|\mathcal{I}|}\sum_{i \in \mathcal{I}} \nabla f_i(w)  - \nabla \overline{f}(w) \right \| \le \sigma \delta', \label{eq:stab_mean}
    \end{equation}
    \begin{equation}
      \left \| \frac{1}{|\mathcal{I}|}\sum_{i \in \mathcal{I}} (\nabla f_i(w) - \nabla \overline{f}(w))(\nabla f_i(w) - \nabla \overline{f}(w))^{\top}
      - \sigma^2 I \right \| \le \sigma^2 \delta'^2 / \epsilon'. \label{eq:stab_cov}
    \end{equation}
  \end{subequations}
  By \(\beta\)-smoothness of \(f_i - \overline{f}\), we have
  \begin{equation}
    \begin{aligned}
      \left \| \frac{1}{|\mathcal{I}|}\sum_{i \in \mathcal{I}} \nabla f_i(w')  - \nabla \overline{f}(w') \right \|
      & \le \left \| \frac{1}{|\mathcal{I}|}\sum_{i \in \mathcal{I}} \left(\nabla f_i(w')  - \nabla f_i(w)\right) \right \| + \left \| \nabla f_i(w) - \nabla \overline{f}(w) \right \| \\
      & \le \beta \|w' - w\| + \sigma \delta'.
    \end{aligned}
  \end{equation}
  Therefore,~\eqref{eq:stab_mean} holds (up to a constant factor) for all \(w'\) such that \(\|w - w'\| \le \sigma \delta' / \beta\).
  
  Next, \cref{eq:stab_cov} is equivalent to the following: for any unit vector \(v\), we have
  \[
    \left | \frac{1}{|\mathcal{I}|}\sum_{i \in \mathcal{I}} (v \cdot (\nabla f_i(w) - \nabla \overline{f}(w)))^2 - \sigma^2 \right | \le \sigma^2 \delta'^2 / \epsilon'.
  \]
  By \(L\)-Lipschitzness and \(\beta\)-smoothness of \(f_i - \overline{f}\), for any unit vector \(v\), we have
  \begin{equation}
    \begin{aligned}
      & \left| \left\{v \cdot \left(\nabla f_i(w)-\nabla \overline f(w)\right)\right\}^2 - \left\{v\cdot \left(\nabla f_i(w')-\nabla \overline f(w')\right)\right\}^2 \right| \\
      & \phantom{\qquad} = \left\{v \cdot \left(\nabla f_i(w)-\nabla \overline f(w)\right) + v\cdot \left(\nabla f_i(w')-\nabla \overline f(w')\right)\right\} \\
      & \phantom{\qquad =} \cdot \left\{v\cdot \left(\nabla f_i(w)-\nabla \overline f(w)\right)- v \cdot \left(\nabla f_i(w')-\nabla \overline f(w')\right)\right\} \le 2 L \cdot \beta \| w - w'|,
    \end{aligned}
  \end{equation}
  It follows that~\eqref{eq:stab_cov} holds (up to a constant factor) for \(w'\) such that \(\|w - w'\| \le \sigma^2 \delta'^2 / (\epsilon' L \beta)\).

  Let \(\xi = \min \left(\sigma \delta' / \beta, \sigma^2 \delta'^2 / (\epsilon' L \beta)\right)\). Then, for all \(w'\) such that \(\|w - w'\| \le \xi\), \( \{ \nabla f_i(w') \}_{i \in \mathcal{I}} \) is \((2\epsilon', 2\delta')\)-stable with respect to \(\nabla \overline{f}(w')\) and \(\sigma^2\).
  It suffices to choose a \(\xi\)-net \(\mathcal{C}\) of \(\mathcal{W}\), where the optimal size of the net is \(|\mathcal{C}| = O((D / \xi)^d)\), and choose \(\tau' = \tau / |\mathcal{C}|\).
  By union bound, with probability at least \(1 - |\mathcal{C}| \tau'\), the stable subset exists for all \(w \in \mathcal{C}\) simultaneously. Since we have argued that for fixed \(w\), the same stable subset applies for all \(w'\) within distance \(\xi\) from \(w\), the subset stability holds simultaneously for all \(w\) with probability at least \(1 - \tau\), as claimed.
\end{proof}

\begin{proof}[Proof of \cref{thm:robust-gd}]
  Combining \cref{lem:stable_subset_gd} and \cref{lem:robust_mean_est}, with probability at least \(1 - \tau\), in each iteration, we can estimate the gradient up to a bias as follows:
  \[
    \| \tilde g(w_t) - \nabla \overline f(w_t) \| = O(\sigma \cdot \delta(\tau')) = \tilde O\left(\sigma \sqrt{\epsilon} + \sigma \sqrt{d \log(1/\tau)/n}\right).
  \]
  Note that the probability is simultaneously for all \(w\), so it does not matter how many iterations we run.
  Conditioned on the gradient estimation bias bound being held, the excess risk bound then follows by applying \cref{lem:biased_gd_smooth_convex} for smoothness loss, or \cref{lem:biased_gd_lipz} for Lipschitz loss with corresponding choices of stepsizes and large enough \(T\).
\end{proof}

\section{Projected Biased Gradient Descent}
In this section, we analyze the convergence of the projected gradient descent algorithm with a biased gradient estimator. We assume the loss function is convex throughout this section.

Both \cref{alg:robust_net} and \cref{alg:robust_gd} can be viewed as instances of the following algorithm.
\begin{algorithm}[H]
  \caption{Projected Gradient Descent with Biased Gradient Estimator}\label{alg:pgd_biased}
  \begin{algorithmic}[1]
    \State \textbf{Input:} Convex function \(F\), stepsize parameters \(\{\eta_t\}_{t \in [T]}\), biased gradient estimator \(\texttt{BiasedEstimator}(w)\), feasible set \(\mathcal{W}\)
    \State Initialize \(w_0 \in \mathcal{W}\) and \(t = 1\).
    \For{\(t \in [T]\)}
    \State Let \(\tilde g_t = \texttt{BiasedEstimator}(w_t)\).
    \State \(w_t \gets \Pi_{\mathcal{W}}(w_{t-1} - \eta_t \tilde g_t)\).
    \EndFor
    \State \textbf{Output:} \(\hat w_T = \frac{1}{T} \sum_{t=1}^T w_t\).
  \end{algorithmic}
\end{algorithm}

Here, \(\Pi_{\mathcal{W}}(\cdot)\) denotes the projection operator onto the feasible set \(\mathcal{W}\), that is,
\[
  \Pi_{\mathcal{W}}(y) = \underset{w \in \mathcal{W}}{\operatorname{\arg\,\min}} \|w - y\|^2.
\]

The projection operation ensures that the iterates \(w_t\) remain within the feasible set \(\mathcal{W}\) throughout the optimization process. The projection step is crucial when the optimization problem is constrained, as it guarantees that the updates do not violate the constraints defined by \(\mathcal{W}\).
We analyze the convergence of the algorithm for (1) smooth loss, (2) Lipschitz loss. For convenience, we always write \(\tilde g_t = g_t + b_t\), where \(g_t\) is the true gradient and \(b_t\) is the bias for the \texttt{BiasedEstimator}. We assume that the bias term is bounded, i.e., \(\|b_t\| \le B\), for all iterations \(t\).

We will need the following property of the projection operator.
\begin{lemma}\label{lem:proj}
  Let \(w \in \mathcal{W}\) and \(y \in \mathbb{R}^d\). We have
  \[
    \left(\Pi_{\mathcal{W}}(y) - y\right)^\top \left(w - \Pi_{\mathcal{W}}(y) \right) \ge 0.
  \]
\end{lemma}

\subsection{Smooth Loss}
\begin{lemma}\label{lem:biased_gd_smooth_convex}
  Suppose \(F\) is \(\beta\)-smooth. Running \cref{alg:pgd_biased} with constant step size \(\eta = \frac{1}{\beta}\), we have
  \begin{equation}
    F\left(\frac{1}{T} \sum_{t=1}^{T} w_t\right) - F(w^*)
    \le \frac{\beta D^2}{2T} + BD.
  \end{equation}
\end{lemma}

\begin{proof}
  By convexity, we have
  \begin{equation}\label{eq:convex}
    F(w_t) \le F(w^*) + g_t^\top (w_t - w^*).
  \end{equation}

  By \(L\)-smoothness, we have
  \begin{equation}\label{eq:smooth}
    F(w_{t+1}) \le F(w_t) + g_t^\top (w_{t+1} - w_t) + \frac{\beta}{2} \| w_t - w_{t+1} \|^2.
  \end{equation}
  Using \cref{lem:proj}, we have
  \begin{equation}\label{eq:proj}
    (w_{t+1} - w_t + \eta_t \tilde g_t)^\top (w^* - w_{t+1}) \ge 0.
  \end{equation}
  We break the left hand into two terms \((w_{t+1} - w_t)^\top (w^* - w_{t+1})\) and \(\eta_t \tilde g_t^\top (w^* - w_{t+1})\). We can write the first term as
  \begin{equation}
    (w_{t+1} - w_t)^\top (w^* - w_{t+1}) = \tfrac{1}{2} \left(
      \|w_t - w^*\|^2 - \|w_{t+1} - w_t\|^2 - \|w_{t+1} - w^*\|^2
    \right)
  \end{equation}
  For the second term, we have
  \begin{equation}
    \eta_t \tilde g_t^\top (w^* - w_{t+1}) = \eta_t g_t^\top (w^* - w_t) + \eta_t g_t^\top(w_t - w_{t+1}) + \eta_t b_t^\top (w^* - w_{t+1})
  \end{equation}
  Using~\eqref{eq:convex},~\eqref{eq:smooth}, and Cauchy-Schwarz inequality to bound the three terms respectively, we have
  \begin{equation}
    \eta_t \tilde g_t^\top (w^* - w_{t+1}) \le \eta_t (F(w^*) - F(w_t)) + \eta_t (F(w_t) - F(w_{t+1})) +
    \tfrac{L \eta_t}{2} \|w_t - w_{t+1}\|^2 + \eta_t B D.
  \end{equation}
  Now going back to~\eqref{eq:proj}, we can combine the above inequalities and choose \(\eta_t = 1/\beta\) to get
  \begin{equation}
    F(w_{t+1}) - F(w^*) \le \tfrac{\beta}{2}
    \|w_t - w^*\|^2 - \tfrac{\beta}{2} \|w_{t+1} - w^*\|^2 + B D
  \end{equation}
  Summing over \(t=0,\ldots, T-1\) and divided by \(T\), we have
  \begin{equation}
    \begin{aligned}
      \frac{1}{T} \sum_{t=1}^{T} \left( F(w_t) - F(w^*) \right)
      & \le \frac{\beta}{2T} \left(\|w_0 - w^*\|^2 - \|w_{T} - w^*\|^2\right) + BD \\
      & \le \frac{\beta D^2}{2T} + BD.
    \end{aligned}
  \end{equation}
  The result then follows by convexity.
\end{proof}

\subsection{Lipschitz Loss}
Alternatively, we can consider the case where the loss function \(F(w)\) is convex and \(L\)-Lipschitz. The following lemma holds.

\begin{lemma}\label{lem:biased_gd_lipz}
  Suppose \(F\) is \(L\)-Lipschitz. Running \cref{alg:pgd_biased} with constant step size \(\eta = \frac{1}{\beta}\), we have
  \begin{equation}
    F\left(\frac{1}{T} \sum_{t=1}^T w_t\right) - F(w^*) \le \frac{DL}{\sqrt{T}} + \left(\frac{1}{\sqrt{T}} + 1\right) B D.
  \end{equation}
\end{lemma}

\begin{proof}
  Let us denote \(y_{t+1} = w_t - \eta_t \tilde g_t\). Using \cref{lem:proj}, we have
  \[
    \|w_t - w^*\| \le \|y_t - w^*\|.
  \]
  By the update rule, we have
  \begin{equation}
    \begin{aligned}
      \tilde g_t (w_t - w^*) & = \frac{1}{\eta} (w_t - y_{t+1})^\top (w_t - w^*) \\
      & \le \frac{1}{2\eta} \left(
        \|w_t - w^*\|^2 - \|w_t - y_{t+1}\|^2 - \|y_{t+1} - w^*\|^2
      \right) \\
      & \le \frac{1}{2\eta} \left(
        \|w_t - w^*\|^2 - \|w_t - y_{t+1}\|^2 - \|w_{t+1} - w^*\|^2
      \right) \\
      & = \frac{1}{2\eta} \left(
        \|w_t - w^*\|^2 - \|w_{t+1} - w^*\|^2
      \right) + \frac{\eta}{2} \|\tilde g_t\|^2.
    \end{aligned}
  \end{equation}
  Now by convexity, we have
  \begin{equation}
    F(w_t) - F(w^*) \le g_t^\top (w_t - w^*) = \tilde g_t^\top (w_t - w^*) - b_t^\top (w_t - w^*).
  \end{equation}
  Recall our assumptions on \(g_t\) and \(b_t\). We have \(\| \tilde g_t\|^2 = \| g_t + b_t \|^2 \le (L + B)^2\). Summing over \(t=0,\ldots, T-1\) and divided by \(T\) gives
  \begin{equation}
    \begin{aligned}
      \frac{1}{T} \sum_{t=0}^{T-1} \left( F(w_t) - F(w^*) \right)
      & \le \frac{1}{2\eta T}
      \left(\|w_0 - w^*\|^2 - \|w_{T+1} - w^*\|^2\right) + \frac{\eta}{2}(L+B)^2 + BD \\
      & \le \frac{D^2}{2\eta T} + \frac{\eta}{2}(L+B)^2 + BD.
    \end{aligned}
  \end{equation}
  Choosing \(\eta = \frac{D}{(L+B)\sqrt{T}}\) and using convexity of \(F\) gives the desired result.
\end{proof}

\section{Discussions on the Assumptions}\label{sec:discussion_assump}

\subsection{On the bounded covariance assumption}\label{sec:bounded-covariance-assumption}
Without loss of generality, we can assume \(\sigma \le L\). The reason is as follows: By Lipschitzness, we have \(\|\nabla f(w) - \nabla \overline{f}(w)\| \le L\) almost surely. By Cauchy-Schwarz Inequality, we have \(\mathbf{E}_{f \sim p^*}[(v \cdot (\nabla f(w) - \nabla \overline{f}(w)))^2] \le L^2\) holds for any unit vector \(v\).
On the other hand, \(L\) can be as larger as \(\sqrt{d} \cdot \sigma\) (e.g. consider standard multivariate normal).

The condition \(\mathbf{E}[(v \cdot (\nabla f(w) - \nabla \overline{f}(w)))^2] \le \sigma^2\) for every unit vector \(v\) is equivalent to requiring that the covariance matrix \(\Sigma_w\) of the gradients \(\nabla f(w)\) satisfies \(\Sigma_w \preceq \sigma^2 I\).
\begin{proposition}
  Let \(\Sigma_w\) denote the covariance matrix of the gradients \(\nabla f(w)\).
  For given \(w\), the following two assumptions are equivalent:
  \begin{enumerate}
    \item For every unit vector \(v\), we have
      \(\mathbf{E}_{f \sim p^*}[(v \cdot (\nabla f(w) - \nabla \overline{f}(w)))^2] \le \sigma^2\).
    \item The covariance matrix satisfies \(\Sigma_w\) \(\preceq \sigma^2 I\).
  \end{enumerate}
  Furthermore, since \(\Sigma_w\) is positive semidefinite, by definition of the spectral norm, the latter assumption can be equivalently written as \(\|\Sigma_w\| \le \sigma^2\).
\end{proposition}

\begin{proof}
  By definition,
  \begin{equation}
    \Sigma_w = \mathbf{E}_{f \sim p^*}[(\nabla f(w) - \nabla \overline{f}(w))(\nabla f(w) - \nabla \overline{f}(w))^{\top}].
  \end{equation}
  We have
  \begin{equation}
    \begin{aligned}
      \mathbf{E}_{f \sim p^*}\left[(v \cdot (\nabla f(w) - \nabla \overline{f}(w)))^2\right]
      & = \mathbf{E}_{f \sim p^*}\left[v \cdot (\nabla f(w) - \nabla \overline{f}(w)) (\nabla f(w) - \nabla \overline{f}(w))^{\top} \cdot v\right] \\
      & = v^{\top} \cdot \mathbf{E}_{f \sim p^*}\left[(\nabla f(w) - \nabla \overline{f}(w))(\nabla f(w) - \nabla \overline{f}(w))^{\top}\right] \cdot v \\
      & = v^{\top} \Sigma_w v.
    \end{aligned}
  \end{equation}
  Therefore, the two assumptions are equivalent.
\end{proof}

\subsection{On the regularity condition}\label{sec:regularity-condition}
For technical reasons, we need to assume regularity conditions such that we can exchange the gradient and expectation, that is, for any \(w\),
\begin{equation}\label{eq:regularity}
  \E_{f \sim p^*} \nabla f(w) = \nabla \overline f(w).
\end{equation}
A necessary condition\footnote{Technically, we need to be more precise about the distribution \(p^*\) over functions. In this paper, we follow the same convention as used by \citet{diakonikolas2019sever}. To be more concrete, we can just think of \(f\) parameterized by some random variable \(X\) and the distribution \(p^*\) is induced by the distribution of \(X\). See the example.}
due to dominated convergence theorem for the regularity condition is that there exists some functional (mapping of functions) \(g(f)\) such that \(\E_{f \sim p^*} [g(f)] < \infty\), and for all \(w\), \( \| \nabla f(w) \| \le g(f) \) almost surely.

Consider the following example where \(f\) takes the form \(f_X(w) = \tfrac{1}{2}(w'X)^2\) where the distribution of \(X\) induces the distribution \(p^*\), and \(X\) is unbounded but has finite second moment, that is, \(\E_X\left[\|XX'\|_2\right] \le M\) for some \(M > 0\).
Note that \(\nabla f_X(w) = XX'w\). In this case, we can take \(g(X) = 2D \cdot \|XX'\|\) and we have that \(\E_X [g(X)] \le \|XX'\|_2 \cdot \|w\| \le  CM < \infty\).
So we have \(\| \nabla f_X(w) \| \le g(X)\) almost surely. In this case, we can exchange the order of expectation and gradient.

\subsection{Comparing bounded covariance assumption with bounded variance assumption}\label{sec:compare-covariance-variance}
Net-based approaches (e.g.\ uniform convergence) often suffer from suboptimal error \citep{feldman2016generalization}. However, \cref{alg:robust_net} indeed achieves the minimax-optimal rate.
We believe the reason is due to the bounded covariance assumption \(\Sigma \preceq \sigma^2 I\). Below, we provide a discussion on the bounded covariance assumption and compare it with the bounded variance assumption.

The bounded covariance assumption \(\Sigma \preceq \sigma^2 I\) is different from the bounded variance assumption \(\mathbf{E}\left \| \nabla f(w) - \nabla \overline{f}(w) \right \|^2 \le \Phi^2\) as commonly used in optimization literature without corruption. Using the property \(\trace{(AB)} = \trace{(BA)}\), this is equivalent to \(\trace{(\Sigma)} \le \Phi^2\).

We comment that neither assumption implies the other. For isotropic Gaussian distribution, where the covariance matrix is \(\Sigma = \sigma^2 I\), we have \(\trace{(\Sigma)} = d \sigma^2\). On the other hand, consider the distribution where the variance is concentrated in one direction, i.e., \(\Sigma = \Phi^2 \cdot v v^{\top}\) for some unit vector \(v\). We have \(\trace{(\Sigma)} = \Phi^2\) and \(\|\Sigma\| = \Phi^2\).
In general, we only know that \(\|\Sigma\| \le \trace{(\Sigma)} \le d \|\Sigma\|\).

Recall \cref{lem:stable_subset}. The complete version of the lemma is as follows:
\begin{lemma}[\citet{diakonikolas2020outlier}]
  Fix any \(0< \tau<1\).
  Let \(S\) be a multiset of \(n\) i.i.d.\ samples from a distribution on \(\R^d\)
  with mean \(\mu\) and covariance \(\Sigma\).
  Let \( \epsilon' =  \Theta(\log(1/ \tau')/n + \epsilon) \leq c\), for a sufficiently small constant \(c>0\).
  Then, with probability at least \(1 - \tau\), there exists a subset \( S' \subseteq S\) such that
  \( |S'| \geq (1 - \epsilon')n\) and \(S'\) is \( (2\epsilon', \delta(\tau'))\)-stable with respect to \(\mu\) and \(\|\Sigma\|\), where
  \(\delta(\tau') = O(\sqrt{ (\srank(\Sigma) \log \srank(\Sigma)) / n} +  \sqrt{\epsilon} + \sqrt{\log(1/\tau')/n})\). Here we use \(\srank(M)\) to denote the stable rank (or intrinsic dimension) of a positive semidefinite matrix \(M\), i.e., \(\srank(M):= \trace(M)/ \|M\| \).
\end{lemma}

Following identical proof steps (recall proofs for our algorithms), we can express our excess risk bound in terms of the covariance matrix \(\Sigma\) as follows:
\begin{equation}
  D \cdot \tilde O\left(\sqrt{\|\Sigma\|\epsilon} + \sqrt{\trace{(\Sigma)} / n} + \sqrt{d \|\Sigma\|\log(1/\tau)/n}\right).
\end{equation}
In our paper, we consider the bounded covariance assumption \(\Sigma \preceq \sigma^2 I\), which is a standard assumption in robust optimization literature. Otherwise, we cannot control the error term \(\sqrt{\|\Sigma\|\epsilon}\) due to corruption.
In the worse case (e.g.\ isotropic Gaussian), we have \(\|\Sigma\| = \sigma^2\) and \(\trace{(\Sigma)} = d \sigma^2\), so the bound reduces to
\begin{equation}
  \sigma \cdot O\left(\sqrt{\epsilon} + \sqrt{d/n} + \sqrt{d \log (1/\tau)/n}\right).
\end{equation}
We see that the second term already contains the dependence on \(d\). Therefore, the \(d\) factor in the last term due to our net-based approach in conjunction with the use of the union bound over the net points, does not affect the rate.

\section{An exponential time algorithm that achieves the minimax-optimal excess risk bound without \texorpdfstring{\(\log d\)}{log(d)} factor}\label{app:exponential-time-alg}
Both of our two algorithms achieve the minimax-optimal excess risk bound up to logarithmic factors. In this section, we show that the minimax-optimal excess risk bound can be achieved without the \(\log d\) factor, but at the cost of exponential time complexity.
Based on \cref{lem:stable_subset},
we can remove the \(\log d\) factor when estimating the gradients, by using the following framework, as shown in~\citet{diakonikolas2020outlier}.
\begin{enumerate}
  \item Set \(k = \lfloor \epsilon' n \rfloor\). Randomly partition \(n\) samples \(S\) into \(k\) buckets of size \(\lfloor n/k \rfloor\) (remove the last bucket if \(n\) is not divisible by \(k\)).
  \item Compute the empirical mean within each bucket and denote the means as \(z_1, \ldots, z_k\).
  \item Run stability-based robust mean estimation over the set \(\{z_1, \ldots, z_k\}\).
\end{enumerate}
Here, the first two steps serve as preprocessing before feeding the data into the robust mean estimation algorithm.
We now restate the robust estimation result without \(\log d\) factor below.
\begin{lemma}\label{lem:mean_estimation_bounded_cov}
  Let \(S\) be an \(\epsilon\)-corrupted set of \(n\) samples from a distribution in \(\R^d\)
  with mean \(\mu\) and covariance \(\Sigma \preceq \sigma^2 I\).
  Let \(\epsilon' =  \Theta(\log(1/\tau)/n + \epsilon) \leq c\) be given, for a constant \(c > 0\).
  Then any stability-based algorithm on input \(S\) and \(\epsilon'\), efficiently computes \(\widehat{\mu}\) such that with probability at least
  \(1 - \tau\), we have
  \( \|\widehat{\mu} - \mu\| = \sigma \cdot O\left(\sqrt{\epsilon} + \sqrt{d/n} + \sqrt{\log(1/\tau)/n}\right)\).
\end{lemma}

We recall that our efficient implementation using grid points cost a \(\log d\) factor due to the suboptimal net size. Using a net with a size matching the covering number \(O((D/\xi)^d)\) will remove the \(\log d\) factor, but at the cost of exponential time complexity for constructing the net and finding a point within \(O(\xi)\) distance for a given point.

Following the same proof steps, as in \cref{app:proof_alg1}, we can derive the excess risk bound without the \(\log d\) factor, at the cost of exponential time complexity.

\section{Dealing with Unknown \texorpdfstring{\(\sigma\)}{sigma}}\label{app:unknown-sigma}
We adapt \cref{alg:robust_net} to work without knowing \(\sigma\) by first getting a lower bound on \(\sigma\) using the filtering algorithm (\cref{alg:filtering}) and then using this lower bound to set the fineness parameter \(\xi\) of the net in \cref{alg:robust_net}.

The modified algorithm is as follows: (1) Estimate \(\sigma\): Choose a point \(w\) and run \cref{alg:filtering} with input \(S = \{ \nabla f_i(w) \}_{i=1}^n\) to obtain a lower bound \(\hat \sigma\).
(2) Then, we run \cref{alg:robust_net} with \(\xi = \hat \sigma \delta / \bar \beta\).

In \cref{alg:robust_net}, we use \(\sigma\) only to determine the fineness of the net via \(\xi = \sigma \sqrt{\epsilon} / \bar \beta\). A smaller \(\xi\) results in a finer net and consequently reduces the error when evaluating gradients at the net point \(w'\) instead of \(w\), that is,~\eqref{eq:smooth_error} still holds with a smaller \(\xi\). Since the excess risk depends on \(\xi\) only through logarithmic terms, the same analysis (see \cref{app:proof_alg1}) holds with a smaller \(\xi\).
It then suffices to choose \(\xi = \hat \sigma \delta / \bar \beta\), where \(\hat \sigma\) is a lower bound on \(\sigma\). We also need to choose \(T = \tilde \Omega\left(\frac{\beta D}{\hat \sigma \sqrt{\epsilon} + \hat \sigma \sqrt{\frac{d \log (1/\tau)}{n}}}\right)\) where we use \(\hat \sigma\) in place of \(\sigma\).
When using smoothing to handle nonsmooth losses, we can choose \(\beta\) similarly by replacing \(\sigma\) with \(\hat \sigma\).

Recall that \cref{alg:filtering} works even when \(\sigma\) is unknown. Moreover, the output \(h\) satisfies \(\|\Sigma(h)\| \le \sigma^2 (1+O(\delta^2/\epsilon))\) (see \citet{diakonikolas2020outlier}). It follows that we can use \(\|\Sigma(h)\|\) to obtain a lower bound on \(\sigma\).
Using \cref{lem:stable_subset}, at any fixed \(w\), we can run \cref{alg:filtering} with input \(S = \{ \nabla f_i(w) \}_{i=1}^n\) to obtain a lower bound \(\hat \sigma\) on \(\sigma\). We have that (plugging in \(\delta(\tau')\) in \cref{lem:stable_subset}), with probability at least \(1 - \tau'\),
\begin{equation}
  \|\Sigma(h)\| \le \sigma^2 (1+O(\delta^2/\epsilon)),
\end{equation}
where \(\delta = \tilde O\left(\sqrt{\epsilon} + \sqrt{d/n} + \sqrt{\log(1/\tau')/n}\right)\).
Therefore, \(\hat \sigma := \sqrt{\|\Sigma(h)\|} / \sqrt{1+O(\delta^2/\epsilon)}\) is a lower bound on \(\sigma\) with probability at least \(1 - \tau'\).

\end{document}